\documentclass[11pt]{article}

\usepackage[margin=1in]{geometry}
\usepackage{amsmath,amssymb,amsthm,mathtools}
\usepackage{enumitem}
\usepackage{booktabs}
\usepackage{array}
\usepackage{tabularx}
\usepackage{aliascnt}
\usepackage{microtype}
\usepackage[numbers,sort&compress]{natbib}
\usepackage[colorlinks=true,linkcolor=blue,citecolor=blue,urlcolor=blue]{hyperref}
\usepackage[nameinlink,noabbrev]{cleveref}
\newtheorem{theorem}{Theorem}[section]

\newaliascnt{proposition}{theorem}
\newtheorem{proposition}[proposition]{Proposition}
\aliascntresetthe{proposition}

\newaliascnt{lemma}{theorem}
\newtheorem{lemma}[lemma]{Lemma}
\aliascntresetthe{lemma}

\newaliascnt{corollary}{theorem}

\aliascntresetthe{corollary}

\newaliascnt{assumption}{theorem}
\newtheorem{assumption}[assumption]{Assumption}
\aliascntresetthe{assumption}

\theoremstyle{definition}
\newaliascnt{definition}{theorem}
\newtheorem{definition}[definition]{Definition}
\aliascntresetthe{definition}

\theoremstyle{remark}
\newaliascnt{remark}{theorem}

\aliascntresetthe{remark}

\numberwithin{equation}{section}

\crefname{theorem}{Theorem}{Theorems}
\Crefname{theorem}{Theorem}{Theorems}
\crefname{lemma}{Lemma}{Lemmas}
\Crefname{lemma}{Lemma}{Lemmas}
\crefname{assumption}{Assumption}{Assumptions}
\Crefname{assumption}{Assumption}{Assumptions}
\crefname{definition}{Definition}{Definitions}
\Crefname{definition}{Definition}{Definitions}
\crefname{remark}{Remark}{Remarks}
\Crefname{remark}{Remark}{Remarks}
\crefname{proposition}{Proposition}{Propositions}
\Crefname{proposition}{Proposition}{Propositions}
\crefname{corollary}{Corollary}{Corollaries}
\Crefname{corollary}{Corollary}{Corollaries}
\crefname{appendix}{Appendix}{Appendices}
\Crefname{appendix}{Appendix}{Appendices}

\DeclareMathOperator{\prox}{prox}

\DeclareMathOperator{\tr}{tr}

\DeclareMathOperator{\Var}{Var}

\newcommand{\R}{\mathbb{R}}
\newcommand{\cP}{\mathcal{P}}

\newcommand{\Id}{I}
\newcommand{\dd}{\,\mathrm{d}}
\newcommand{\law}{\mathcal{L}}
\newcommand{\E}{\mathbb{E}}

\newcommand{\norm}[1]{\left\lVert #1\right\rVert}
\newcommand{\ip}[2]{\left\langle #1,#2\right\rangle}
\newcommand{\W}{W_2}

\newcommand{\Pp}{\mathbb{P}}

\newcommand{\op}{\mathrm{op}}

\providecommand{\Pcal}{\cP}
\providecommand{\Wtwo}{\W}
\providecommand{\Law}{\law}
\DeclareMathOperator{\sym}{sym}
\newcommand{\abs}[1]{\left\lvert #1\right\rvert}

\allowdisplaybreaks[3]
\title{\bfseries Poisson-Corrector Complexity Bounds for\\
Moreau--Yosida Unadjusted Langevin Sampling}
\author{Yuchen Xin\textsuperscript{*},
Zhihua Zhang\textsuperscript{$\dagger$}}
\date{}

\begin{document}
\maketitle

\begingroup
\renewcommand{\thefootnote}{\fnsymbol{footnote}}
\footnotetext[1]{School of Mathematical Sciences, Peking University; email: \texttt{2301110087@pku.edu.cn}}
\footnotetext[2]{School of Mathematical Sciences, Peking University; email: \texttt{zhzhang@math.pku.edu.cn}}
\endgroup

\begin{abstract}
We study the classical Moreau--Yosida unadjusted Langevin algorithm (MYULA) for $\pi(\dd x)\propto e^{-f(x)-g(x)}\dd x$, where $f\in C^2(\R^d)$ is $m$-strongly convex with $L_f$-Lipschitz gradient and $g:\R^d\to\R$ is convex and globally $G$-Lipschitz.
For the Moreau-smoothed target $\pi_\lambda$ and the MYULA invariant law $\widehat\pi_{\lambda,h}$, we prove
\[
  \sqrt m\,\Wtwo(\pi_\lambda,\widehat\pi_{\lambda,h})
  =O(h)+\widetilde O(h^{3/4})
\]
under $0<h(L_f+\lambda^{-1})\le c$, with only logarithmic dependence on $\lambda^{-1}$ in the error coefficients. Combining this estimate with the Moreau approximation bias yields $\widetilde O(\varepsilon^{-4/3})$ iterations to achieve $\sqrt m\,\Wtwo(\mu_N,\pi)\le\varepsilon$, for fixed model parameters and initialization.
The proof combines a discrete Poisson corrector with active-trace estimates and a shared-noise bound for the exact--Euler two-point curvature.
\end{abstract}

\tableofcontents

\section{Introduction}\label{sec:introduction}

Sampling from high-dimensional distributions is a fundamental problem in statistics and machine learning.
We consider distributions of the form
\[
  \pi(\dd x)\propto \exp\{-f(x)-g(x)\}\dd x,
  \qquad x\in\R^d,
\]
where $f\in C^2(\R^d)$ is $m$-strongly convex with $L_f$-Lipschitz gradient, and $g:\R^d\to\R$ is convex and globally $G$-Lipschitz, with $G>0$. The function $g$ need not be differentiable; examples include $\ell_1$ and total-variation penalties \citep{DurmusMoulinesPereyra2018,MouEtAl2022}.

The Moreau--Yosida unadjusted Langevin algorithm (MYULA) \citep{DurmusMoulinesPereyra2018} replaces $g$ by its Moreau envelope $g_\lambda$ and applies an Euler step to the corresponding Langevin diffusion:
\[
  \widehat X_{k+1}
  =\widehat X_k-h\nabla(f+g_\lambda)(\widehat X_k)
   +\sqrt{2h}\,\xi_{k+1},
  \qquad \xi_{k+1}\stackrel{\mathrm{i.i.d.}}{\sim}N(0,\Id).
\]
Since
$\nabla g_\lambda(x)=\lambda^{-1}
(x-\prox_{\lambda g}(x))$,
each iteration requires one gradient evaluation of $f$ and one
proximal evaluation of $g$. This makes MYULA a natural sampling
method when the same proximal operations are already available
for optimization.

The main difficulty is the interaction between smoothing and
discretization. The smoothed target
$\pi_\lambda(\dd x)\propto e^{-f(x)-g_\lambda(x)}\dd x$
approaches $\pi$ as $\lambda$ decreases, but the global smoothness
bound
\[
  L_\lambda=L_f+\lambda^{-1}
\]
increases. Moreover, the invariant law of the Euler scheme generally
differs from $\pi_\lambda$. Accuracy for the original target therefore
requires controlling three errors: the Moreau approximation bias,
the Euler invariant-measure bias, and convergence from the initial
distribution. A fixed-$\lambda$ estimate is useful for this purpose
only when its dependence on $\lambda$ is also quantified.

Two recent works provide the main ingredients for our analysis.
\citet{XinZhang2026} show that averaged weak Moreau curvature can
replace a worst-case curvature bound in the analysis of MYULA.
\citet{PedrottiWhalley2026} use a discrete Poisson decomposition
to obtain a linear-in-$h$ invariant-measure bias for smooth,
strongly log-concave targets. Direct specialization of their
global-smoothness bound to $f+g_\lambda$, combined with
$\lambda\asymp\varepsilon/G^2$, gives a sufficient
$\widetilde O(\varepsilon^{-2})$ MYULA complexity; the calculation
is given in Section~\ref{sec:related}.
We obtain a sharper dependence on accuracy by controlling the
curvature terms within the corrector argument.

\paragraph{Main result.}
Let $\widehat\pi_{\lambda,h}$ be the invariant law of MYULA.
Under $0<hL_\lambda\le c$, Theorem~\ref{thm:fixed-lambda} gives
\[
  \sqrt m\,\Wtwo(\pi_\lambda,\widehat\pi_{\lambda,h})
  \le C_1h+C_2h^{3/4}
  \left[
    1+\log\!\left(e+\frac{\sqrt h}{G\sqrt d\,\lambda}\right)
  \right]^2,
\]
where $C_1,C_2$ depend on the model parameters but not on
$h$ or $\lambda$. Their parameter dependence is explicit in
the theorem. Thus, the error coefficients depend on
$\lambda^{-1}$ only logarithmically, although the usual
step-size restriction remains.

Combining this estimate with
\[
  \sqrt m\,\Wtwo(\pi_\lambda,\pi)\le \frac{G^2\lambda}{4}
\]
from \citet[Proposition~5.9]{XinZhang2026}, we choose
$\lambda=\varepsilon/G^2$ and the step size specified in
Theorem~\ref{thm:complexity}.
For fixed model parameters and a fixed initial law
$\mu_0\in\Pcal_2(\R^d)$, the resulting sufficient iteration
complexity is
\[
  N(\varepsilon)=\widetilde O(\varepsilon^{-4/3})
  \qquad\text{for}\qquad
  \sqrt m\,\Wtwo(\mu_N,\pi)\le\varepsilon,
\]
where $\mu_N$ is the law of the $N$th iterate and
$\widetilde O$ suppresses logarithmic factors.

\paragraph{Proof approach.}
We follow the Poisson-corrector framework of \citet{PedrottiWhalley2026}, comparing MYULA with the exact Langevin diffusion for $\pi_\lambda$ driven by the same noise.
The corrector captures cancellations between local errors at successive steps.
Our refinement combines the active-trace estimates of \citet{XinZhang2026} with a shared-noise curvature estimate for the coupled processes. This allows us to control average curvature rather than use the worst-case bound $\lambda^{-1}$ for $g_\lambda$ throughout the error estimate.
Together with contraction from strong convexity, these estimates yield the stated invariant-measure bias bound.

The algorithm itself is unchanged: $\lambda$ and $h$ are fixed
during each run, and no Metropolis correction is introduced.
Our complexity counts iterations with exact proximal evaluations,
not the internal cost of solving proximal subproblems.
Section~\ref{sec:related} places the result in the literature;
Sections~\ref{sec:preliminaries}--\ref{sec:results} give the
definitions, assumptions, and precise statements.
The proofs occupy Sections~\ref{sec:main-proof}--\ref{sec:shared-noise}.

\section{Related Work}\label{sec:related}

\paragraph{Moreau regularization and proximal sampling.}
\citet{Pereyra2016} developed proximal MCMC methods for nonsmooth
log-concave models, and \citet{DurmusMoulinesPereyra2018} introduced
the classical MYULA construction with asymptotic and nonasymptotic
guarantees. These methods connect sampling with proximal operations
used in convex optimization.
A complementary viewpoint is provided by
\citet{DurmusMajewskiMiasojedow2019}, who interpret Langevin
algorithms through optimization on Wasserstein space and analyze
both smooth and nonsmooth variants.
For a smoothing-based method, bounds for a fixed surrogate target
must be combined with the regularization bias before their
accuracy dependence can be compared with an end-to-end guarantee.

\paragraph{Active-trace analysis of MYULA.}
The closest analysis of the same transition kernel is
\citet{XinZhang2026}.
Their reference heat-path active trace averages the weak Moreau
Hessian trace along the Gaussian substep of an update started
from $\pi_\lambda$. For the structured piecewise-linear,
lasso-type, group, and total-variation penalties treated there,
curvature--tube estimates yield
$\widetilde O(\varepsilon^{-2})$ end-to-end complexity.
We use their weak-curvature and Moreau-bias tools, but replace
the reference-path recursion by a corrected synchronous coupling.
The additional shared-noise estimate controls the actual
exact--Euler two-point curvature for general convex Lipschitz
$g$, without the geometric verification required by those examples.

\paragraph{Poisson correctors and a direct MYULA benchmark.}
For an $m$-strongly convex potential with $L$-Lipschitz gradient,
\citet[Theorem~1]{PedrottiWhalley2026} prove
\[
  \sqrt m\,\Wtwo(\pi,\widehat\pi)\le 6hL\sqrt d,
  \qquad 0<hL\le1.
\]
Specializing this bound to $U_\lambda=f+g_\lambda$ gives
\[
  \sqrt m\,\Wtwo(\pi_\lambda,\widehat\pi_{\lambda,h})
  \le 6hL_\lambda\sqrt d.
\]
Here we use the extension from $C^2$ to $C^{1,1}$ potentials
by mollification.\footnote{Apply the bound to
$f+g_\lambda*\rho_\delta$.
Mollification preserves the bounds $m$ and $L_\lambda$, and the
gradients converge uniformly. Synchronous contraction gives
$W_2$ convergence of both the diffusion and Euler invariant laws,
so the estimate passes to the limit.}
With $\lambda\asymp\varepsilon/G^2$, choosing
$h\asymp\varepsilon/(L_\lambda\sqrt d)$ and using Wasserstein
contraction yield the sufficient iteration bound
\[
  N_{\mathrm{PW}}(\varepsilon)
  =\widetilde O\!\left[
    \frac{\sqrt d}{m}
    \left(
      \frac{L_f}{\varepsilon}
      +\frac{G^2}{\varepsilon^2}
    \right)
  \right].
\]
Thus, direct use of their global-smoothness estimate gives
$\widetilde O(\varepsilon^{-2})$ accuracy dependence for MYULA.
Our contribution is a refinement of the curvature terms in
the corrector analysis, leading to
$\widetilde O(\varepsilon^{-4/3})$ under the assumptions of
this paper. The gain concerns the smoothing dependence,
not the order in $h$ for a fixed smoothed target.

\paragraph{Other composite samplers.}
Different transitions can avoid a separate Moreau approximation of the target.
\citet{SalimRichtarik2020} give a primal--dual analysis of proximal stochastic gradient Langevin sampling, including a $\widetilde O(\varepsilon^{-2})$ Wasserstein complexity in the strongly convex setting.
\citet{MouEtAl2022} combine proximal sampling proposals with a Metropolis--Hastings correction and obtain polylogarithmic dependence on $\varepsilon^{-1}$ for total variation error at most $\varepsilon$, under their assumptions.
Their method requires a proximal sampling oracle and evaluation of the associated normalizing constants, rather than only a deterministic proximal operator.
These results concern different transition kernels and, in the latter case, a different error metric and oracle model.
Our result concerns classical MYULA with exact proximal evaluations and accuracy in $\sqrt m\,W_2$ for the original composite target.

\section{Preliminaries}\label{sec:preliminaries}

\subsection{Basic notation}

All probability measures are defined on the Borel $\sigma$-algebra of
$\R^d$.  The Euclidean norm and inner product are denoted by
$\norm{\cdot}$ and $\ip{\cdot}{\cdot}$.  For
$\mu,\nu\in\Pcal_2(\R^d)$, the quadratic Wasserstein distance is
\begin{equation}\label{eq:W2-def}
  \Wtwo^2(\mu,\nu)
  :=\inf_{\substack{\Law(X)=\mu\\\Law(Y)=\nu}}
  \E\norm{X-Y}^2.
\end{equation}
For symmetric matrices $A,B\in\R^{d\times d}$, the notation $A\preceq B$
means $v^\top Av\le v^\top Bv$ for every $v\in\R^d$.  Constants $c,C>0$
are universal and may change from line to line.  For a measure $\mu$ and a
vector-valued function $F$, we use
$\norm{F}_{L^p(\mu)}=(\int\norm{F(x)}^p\mu(\dd x))^{1/p}$.

\subsection{Moreau--Yosida regularization and weak second-order structure}

Let $g:\R^d\to\R$ be convex. For $\lambda>0$, its Moreau envelope is
\begin{equation}\label{eq:moreau}
  g_\lambda(x)
  :=\inf_{y\in\R^d}
  \left\{g(y)+\frac{\norm{x-y}^2}{2\lambda}\right\}.
\end{equation}
Standard Moreau theory gives
\begin{equation}\label{eq:moreau-gradient}
  \nabla g_\lambda(x)
  =\lambda^{-1}\bigl(x-\prox_{\lambda g}(x)\bigr)
\end{equation}
and the cocoercivity inequality
\begin{equation}\label{eq:cocoercivity}
  \ip{x-y}{\nabla g_\lambda(x)-\nabla g_\lambda(y)}
  \ge
  \lambda\norm{\nabla g_\lambda(x)-\nabla g_\lambda(y)}^2;
\end{equation}
see, for example, \citet[Chapters~12 and~18]{BauschkeCombettes2017}.
Consequently $g_\lambda\in C^{1,1}(\R^d)$ and $\operatorname{Lip}(\nabla g_\lambda)\le\lambda^{-1}$.  If $g$ is $G$-Lipschitz, then
\begin{equation}\label{eq:moreau-gradient-bound}
  \norm{\nabla g_\lambda(x)}\le G,
  \qquad x\in\R^d.
\end{equation}

Because $\nabla g_\lambda$ is Lipschitz, it is differentiable almost
everywhere.  Its almost-everywhere derivative also represents its weak
derivative; we write
\begin{equation}\label{eq:weak-hessian-def}
  H_\lambda:=\nabla^2 g_\lambda
  \quad\text{and}\quad
  a_\lambda:=\tr H_\lambda.
\end{equation}
The matrix field $H_\lambda$ may be chosen measurable and satisfies
\begin{equation}\label{eq:weak-hessian-bounds}
  0\preceq H_\lambda(x)\preceq\lambda^{-1}\Id,
  \qquad
  0\le a_\lambda(x)\le d/\lambda
  \quad\text{for a.e. }x.
\end{equation}
Following \citet{XinZhang2026}, we call $a_\lambda$ the \emph{Moreau active trace}.  The weak-Hessian properties and integration-by-parts identities used below are collected in Appendix~\ref{app:weak-hessian}.

\section{Problem Setup and Standing Assumptions}\label{sec:setup}

We consider the composite target
\begin{equation}\label{eq:target}
  \pi(\dd x)=Z^{-1}\exp\{-f(x)-g(x)\}\dd x.
\end{equation}

\begin{assumption}[Smooth component]\label{ass:smooth}
The function $f\in C^2(\R^d)$, and there are constants
$0<m\le L_f<\infty$ such that
\begin{equation}\label{eq:f-assumption}
  m\Id\preceq\nabla^2f(x)\preceq L_f\Id,
  \qquad x\in\R^d.
\end{equation}
Moreover,
\begin{equation}\label{eq:tau-f}
  \tau_f:=\sup_{x\in\R^d}\tr\nabla^2f(x)<\infty.
\end{equation}
\end{assumption}

\begin{assumption}[Nonsmooth component]\label{ass:nonsmooth}
The function $g:\R^d\to\R$ is convex and globally $G$-Lipschitz for some $G>0$.
\end{assumption}

For $\lambda>0$, define
\begin{equation}\label{eq:pi-lambda}
  U_\lambda:=f+g_\lambda,
  \qquad
  \pi_\lambda(\dd x):=Z_\lambda^{-1}e^{-U_\lambda(x)}\dd x,
  \qquad
  L_\lambda:=L_f+\lambda^{-1}.
\end{equation}
Then $U_\lambda$ is $m$-strongly convex, belongs to $C^{1,1}(\R^d)$, and
$\operatorname{Lip}(\nabla U_\lambda)\le L_\lambda$.

MYULA is the Euler scheme for the Langevin diffusion with potential
$U_\lambda$:
\begin{equation}\label{eq:myula}
  \widehat X_{k+1}
  =\widehat X_k-h\nabla U_\lambda(\widehat X_k)
  +\sqrt{2h}\,\xi_{k+1},
  \qquad
  \xi_{k+1}\stackrel{\mathrm{i.i.d.}}{\sim}N(0,\Id).
\end{equation}
Its transition kernel is denoted by $Q_{\lambda,h}$, and
$\mu_k:=\Law(\widehat X_k)$.  This normalization and the use of
$\sqrt m\,W_2$ agree with the setup in \citet[Section~4]{XinZhang2026}.

\section{Main Results}\label{sec:results}

Define
\begin{equation}\label{eq:R-K}
  R:=\frac{G+\sqrt{G^2+4\tau_f}}{2},
  \qquad
  K:=2G+\frac{L_f}{\sqrt m}(\sqrt d+2),
\end{equation}
and
\begin{equation}\label{eq:ell-h-lambda}
  \ell_{h,\lambda}
  :=1+\log\!\left(e+\frac{\sqrt h}{G\sqrt d\,\lambda}\right).
\end{equation}
The quantity $R$ controls the stationary $L^2$ norm of the drift and the
mean active trace, while $K$ controls all stationary drift moments.

\begin{theorem}[Invariant-measure bias at fixed smoothing]\label{thm:fixed-lambda}
Under Assumptions~\ref{ass:smooth}--\ref{ass:nonsmooth}, there exist universal
constants $c,C>0$ such that the following holds.  For every $\lambda>0$ and
\begin{equation}\label{eq:step-condition}
  0<hL_\lambda\le c,
\end{equation}
the kernel $Q_{\lambda,h}$ has a unique invariant distribution
$\widehat\pi_{\lambda,h}\in\Pcal_2(\R^d)$, and
\begin{align}
  \sqrt m\,\Wtwo(\pi_\lambda,\widehat\pi_{\lambda,h})
  \le{}&
  C\Bigl[R(\sqrt m+\sqrt{L_f})+L_f\sqrt d\Bigr]h
  \label{eq:fixed-main}\\
  &+C\Bigl[G\sqrt R
  +K\sqrt G\,d^{1/4}\ell_{h,\lambda}^{2}\Bigr]h^{3/4}.
  \notag
\end{align}
\end{theorem}

\begin{theorem}[End-to-end complexity]\label{thm:complexity}
Under Assumptions~\ref{ass:smooth}--\ref{ass:nonsmooth}, let
$0<\varepsilon\le1$ and set
\begin{equation}\label{eq:lambda-epsilon}
  \lambda_\varepsilon:=\frac{\varepsilon}{G^2},
  \qquad
  \ell_\varepsilon:=1+\log(e+\varepsilon^{-1}).
\end{equation}
For a sufficiently small universal constant $c>0$, set
\begin{align}
  h_\varepsilon:=c\min\Biggl\{&
  \frac{1}{L_f+G^2/\varepsilon},
  \frac{\varepsilon}
  {R(\sqrt m+\sqrt{L_f})+L_f\sqrt d},
  \label{eq:h-epsilon}\\
  &\left(
  \frac{\varepsilon}
  {G\sqrt R+K\sqrt G\,d^{1/4}\ell_\varepsilon^2}
  \right)^{4/3}
  \Biggr\}.
  \notag
\end{align}
For every $\mu_0\in\Pcal_2(\R^d)$, if
\begin{equation}\label{eq:N-epsilon}
  N\ge
  \frac{C}{mh_\varepsilon}
  \log\!\left(
  2+\frac{\sqrt m\,\Wtwo(\mu_0,\pi)}
  {\varepsilon}
  \right),
\end{equation}
then
\begin{equation}\label{eq:end-to-end-error}
  \sqrt m\,\Wtwo(
  \mu_0Q_{\lambda_\varepsilon,h_\varepsilon}^{N},\pi)
  \le\varepsilon.
\end{equation}
Moreover, the smallest integer satisfying~\eqref{eq:N-epsilon} obeys
\begin{align}
  N\le \frac{C}{m}\Biggl[&
  L_f+\frac{G^2}{\varepsilon}
  +\frac{R(\sqrt m+\sqrt{L_f})+L_f\sqrt d}{\varepsilon}
  \label{eq:N-explicit}\\
  &+\frac{G^{4/3}R^{2/3}
  +K^{4/3}G^{2/3}d^{1/3}\ell_\varepsilon^{8/3}}
  {\varepsilon^{4/3}}
  \Biggr]
  \log\!\left(
  2+\frac{\sqrt m\,\Wtwo(\mu_0,\pi)}
  {\varepsilon}
  \right)+1.
  \notag
\end{align}
In particular, for fixed $m,L_f,\tau_f,G,d$ and fixed $\mu_0$,
\begin{equation}\label{eq:tilde-rate}
  N(\varepsilon)=\widetilde O(\varepsilon^{-4/3}).
\end{equation}
\end{theorem}

\section{Proof of the Main Results}\label{sec:main-proof}

This section proves Theorems~\ref{thm:fixed-lambda} and \ref{thm:complexity} after isolating two technical propositions. The construction and estimates for the corrector are proved in Section~\ref{sec:corrector-proof}; the shared-noise trace estimate and the resulting two-point curvature feedback are proved in Section~\ref{sec:shared-noise}.

\subsection{The two-point drift matrix}

For $x,y\in\R^d$, define
\begin{equation}\label{eq:Af-def}
  A^f(x,y)
  :=\int_0^1\nabla^2f\bigl(y+t(x-y)\bigr)\dd t
\end{equation}
and
\begin{equation}\label{eq:Ag-def}
  A_\lambda^g(x,y):=
  \begin{cases}
  \displaystyle
  \frac{
  \bigl(\nabla g_\lambda(x)-\nabla g_\lambda(y)\bigr)
  \bigl(\nabla g_\lambda(x)-\nabla g_\lambda(y)\bigr)^\top}
  {\ip{x-y}{\nabla g_\lambda(x)-\nabla g_\lambda(y)}},
  &\nabla g_\lambda(x)\ne\nabla g_\lambda(y),\\[4mm]
  0,&\nabla g_\lambda(x)=\nabla g_\lambda(y).
  \end{cases}
\end{equation}
Set
\begin{equation}\label{eq:A-def}
  A_\lambda(x,y):=A^f(x,y)+A_\lambda^g(x,y).
\end{equation}
We refer to $A_\lambda(x,y)$ as the \emph{two-point drift matrix}.  

\begin{lemma}[Exact two-point drift identity]\label{lem:two-point}
For every $x,y\in\R^d$,
\begin{align}
  \nabla U_\lambda(x)-\nabla U_\lambda(y)
  &=A_\lambda(x,y)(x-y),
  \label{eq:two-point-identity}\\
  m\Id\preceq A_\lambda(x,y)
  &\preceq L_\lambda\Id,
  \label{eq:two-point-bounds}\\
  0\preceq A^f(x,y)&\preceq L_f\Id,
  \qquad
  0\preceq A_\lambda^g(x,y)\preceq\lambda^{-1}\Id.
  \label{eq:two-point-split}
\end{align}
\end{lemma}

\begin{proof}
The fundamental theorem of calculus gives
\[
  A^f(x,y)(x-y)=\nabla f(x)-\nabla f(y),
  \qquad
  m\Id\preceq A^f(x,y)\preceq L_f\Id.
\]
Suppose first that
$\nabla g_\lambda(x)\ne\nabla g_\lambda(y)$.  The denominator in
\eqref{eq:Ag-def} is strictly positive by~\eqref{eq:cocoercivity}, and direct
multiplication yields
\[
  A_\lambda^g(x,y)(x-y)
  =\nabla g_\lambda(x)-\nabla g_\lambda(y).
\]
For any $v\in\R^d$, Cauchy--Schwarz and~\eqref{eq:cocoercivity} give
\begin{align*}
  v^\top A_\lambda^g(x,y)v
  &=\frac{
  \ip{v}{\nabla g_\lambda(x)-\nabla g_\lambda(y)}^2}
  {\ip{x-y}{\nabla g_\lambda(x)-\nabla g_\lambda(y)}}\\
  &\le
  \frac{\norm v^2
  \norm{\nabla g_\lambda(x)-\nabla g_\lambda(y)}^2}
  {\lambda
  \norm{\nabla g_\lambda(x)-\nabla g_\lambda(y)}^2}
  =\lambda^{-1}\norm v^2.
\end{align*}
The case of a zero gradient difference is immediate from the definition.
Adding the two components proves all assertions.
\end{proof}

\subsection{A single corrector and the two technical inputs}

Let $(X_t)_{t\ge0}$ be the stationary Langevin diffusion
\begin{equation}\label{eq:diffusion}
  \dd X_t=-\nabla U_\lambda(X_t)\dd t+\sqrt2\dd B_t,
  \qquad X_0\sim\pi_\lambda,
\end{equation}
and let $(P_t)_{t\ge0}$ denote its reversible Markov semigroup.  Throughout
the paper, $\chi_h$ denotes the same mean-zero vector field characterized by
\begin{equation}\label{eq:abstract-poisson}
  \int_{\R^d}\chi_h\dd\pi_\lambda=0,
  \qquad
  (I-P_h)\chi_h
  =\int_0^h(I-P_s)\nabla U_\lambda\dd s.
\end{equation}
The integral in~\eqref{eq:abstract-poisson} is taken in
$L^2(\pi_\lambda;\R^d)$; it is the $L^2$ limit of Riemann sums because
$s\mapsto P_s\nabla U_\lambda$ is $L^2$-continuous.  Existence,
uniqueness, and the bounds needed below are established in
Proposition~\ref{prop:operator-corrector}.  The main proof uses only
\eqref{eq:abstract-poisson} and the estimates stated below.

For $k\ge0$, define the one-step local drift error
\begin{equation}\label{eq:Dk}
  D_k:=\int_0^h
  \bigl\{\nabla U_\lambda(X_{kh})-
  \nabla U_\lambda(X_{kh+s})\bigr\}\dd s
\end{equation}
and the filtration
$\mathcal F_t:=\sigma(X_0,B_s:0\le s\le t)$.

\begin{proposition}[Corrector decomposition and moment bounds]
\label{prop:corrector-input}
Under Assumptions~\ref{ass:smooth}--\ref{ass:nonsmooth}, there are random
vectors $M_{k+1}$ such that
\begin{align}
  D_k&=\chi_h(X_{kh})-\chi_h(X_{(k+1)h})+M_{k+1},
  \label{eq:corrector-decomp}\\
  \E[M_{k+1}\mid\mathcal F_{kh}]&=0,
  \label{eq:M-centered}\\
  \norm{\chi_h}_{L^2(\pi_\lambda)}&\le hR,
  \label{eq:chi-L2-input}\\
  \norm{\chi_h}_{L^p(\pi_\lambda;\R^d)}&\le ChKp^2,
  \qquad p\ge2,
  \label{eq:chi-Lp-input}\\
  \E\norm{M_{k+1}}^2
  &\le C\left\{
  L_f^2(dh^3+R^2h^4)+G^2Rh^{5/2}
  \right\}.
  \label{eq:M-var-input}
\end{align}
\end{proposition}

\begin{proposition}[Two-point curvature feedback]\label{prop:feedback-input}
Couple~\eqref{eq:diffusion} and~\eqref{eq:myula} from
$X_0=\widehat X_0\sim\pi_\lambda$ using the same Brownian motion, and set
\begin{equation}\label{eq:Ak-global}
  A_k:=A_\lambda(X_{kh},\widehat X_k),
  \qquad
  A_k^g:=A_\lambda^g(X_{kh},\widehat X_k).
\end{equation}
If~\eqref{eq:step-condition} holds, then, for every $k\ge0$,
\begin{equation}\label{eq:feedback-input}
  \E\left[
  \chi_h(X_{kh})^\top A_k^g\chi_h(X_{kh})
  \right]
  \le
  CK^2G\sqrt d\,h^{3/2}\ell_{h,\lambda}^{4}.
\end{equation}
\end{proposition}

Proposition~\ref{prop:corrector-input} is proved in Section~\ref{sec:corrector-proof}, and Proposition~\ref{prop:feedback-input} is proved in Subsection~\ref{sec:feedback-proof}.

\subsection{A one-step matrix inequality}

\begin{lemma}[One-step energy inequality]\label{lem:matrix-energy}
Let $A$ be symmetric and satisfy $m\Id\preceq A\preceq L\Id$.  If
$0<hL\le1$, then, for all $u,v\in\R^d$,
\begin{equation}\label{eq:matrix-energy}
  \norm{(\Id-hA)u+hAv}^2
  \le(1-mh)\norm u^2+h\,v^\top Av.
\end{equation}
\end{lemma}

\begin{proof}
This is the matrix estimate used in the global-error recursion of
\citet[Section~2.1, pp.~3--4]{PedrottiWhalley2026}.  For completeness, diagonalize $A$.
For every eigenvalue $a\in[m,L]$, convexity of the square gives
\[
  \abs{(1-ha)u+hav}^2
  \le(1-ha)u^2+hav^2,
  \qquad 0\le ha\le1.
\]
Summing over the eigenbasis and using $a\ge m$ proves~\eqref{eq:matrix-energy}.
\end{proof}

\subsection{Proof of the fixed-smoothing theorem}

\begin{proof}[Proof of Theorem~\ref{thm:fixed-lambda}]
Start the diffusion and MYULA from
$X_0=\widehat X_0\sim\pi_\lambda$, and use the same Gaussian increments, that
is,
\[
  \sqrt h\,\xi_{k+1}=B_{(k+1)h}-B_{kh}.
\]
Define
\begin{equation}\label{eq:Zk}
  Z_k:=X_{kh}-\widehat X_k.
\end{equation}
The exact diffusion satisfies
\begin{equation}\label{eq:exact-step}
  X_{(k+1)h}
  =X_{kh}-h\nabla U_\lambda(X_{kh})
  +\sqrt2(B_{(k+1)h}-B_{kh})+D_k.
\end{equation}
Subtracting~\eqref{eq:myula} from~\eqref{eq:exact-step} and applying
Lemma~\ref{lem:two-point} yields
\begin{align}
  Z_{k+1}
  &=Z_k-h\bigl\{\nabla U_\lambda(X_{kh})-
  \nabla U_\lambda(\widehat X_k)\bigr\}+D_k
  \notag\\
  &=(\Id-hA_k)Z_k+D_k.
  \label{eq:Z-recursion}
\end{align}
Insert~\eqref{eq:corrector-decomp} and define the corrected error
\begin{equation}\label{eq:Yk}
  Y_k:=Z_k+\chi_h(X_{kh}).
\end{equation}
Then
\begin{align}
  Y_{k+1}
  &=(\Id-hA_k)Z_k+\chi_h(X_{kh})+M_{k+1}
  \notag\\
  &=(\Id-hA_k)Y_k+hA_k\chi_h(X_{kh})+M_{k+1}.
  \label{eq:Y-recursion}
\end{align}

By~\eqref{eq:two-point-bounds} and~\eqref{eq:step-condition}, after reducing
$c$ if necessary, Lemma~\ref{lem:matrix-energy} applies.  Moreover,
$Y_k$, $A_k$, and $\chi_h(X_{kh})$ are $\mathcal F_{kh}$-measurable.  Hence
\eqref{eq:M-centered} removes all cross terms involving $M_{k+1}$ after
conditional expectation, and
\begin{align}
  \E\norm{Y_{k+1}}^2
  &=\E\norm{(\Id-hA_k)Y_k+hA_k\chi_h(X_{kh})}^2
  +\E\norm{M_{k+1}}^2
  \notag\\
  &\le(1-mh)\E\norm{Y_k}^2+h\Phi_k+
  \E\norm{M_{k+1}}^2,
  \label{eq:Y-energy}
\end{align}
where
\begin{equation}\label{eq:Phi-k}
  \Phi_k:=\E\left[
  \chi_h(X_{kh})^\top A_k\chi_h(X_{kh})
  \right].
\end{equation}

Since $A_k=A^f(X_{kh},\widehat X_k)+A_k^g$, the bounds
\eqref{eq:two-point-split}, stationarity $X_{kh}\sim\pi_\lambda$,
\eqref{eq:chi-L2-input}, and~\eqref{eq:feedback-input} imply
\begin{align}
  \Phi_k
  &\le L_f\E\norm{\chi_h(X_{kh})}^2
  +\E\left[\chi_h(X_{kh})^\top A_k^g\chi_h(X_{kh})\right]
  \notag\\
  &\le L_fh^2R^2
  +CK^2G\sqrt d\,h^{3/2}\ell_{h,\lambda}^{4}
  =:\overline\Phi_h.
  \label{eq:Phi-bound}
\end{align}
Set
\begin{equation}\label{eq:Mbar}
  \overline M_h
  :=C\left\{L_f^2(dh^3+R^2h^4)+G^2Rh^{5/2}\right\}.
\end{equation}
Then~\eqref{eq:M-var-input} gives
$\E\norm{M_{k+1}}^2\le\overline M_h$.  Iterating~\eqref{eq:Y-energy}, for
$n\ge1$,
\begin{align}
  \E\norm{Y_n}^2
  &\le(1-mh)^n\E\norm{Y_0}^2
  +\sum_{j=0}^{n-1}(1-mh)^{n-1-j}
  (h\overline\Phi_h+\overline M_h)
  \notag\\
  &\le(1-mh)^n\E\norm{Y_0}^2
  +\frac{h\overline\Phi_h+\overline M_h}{mh}.
  \label{eq:Y-iteration}
\end{align}
Letting $n\to\infty$ and substituting~\eqref{eq:Phi-bound} and
\eqref{eq:Mbar} gives
\begin{align}
  m\limsup_{n\to\infty}\E\norm{Y_n}^2
  \le C\bigl\{&
  L_fR^2h^2+L_f^2dh^2+L_f^2R^2h^3
  \label{eq:Y-limsup}\\
  &+K^2G\sqrt d\,h^{3/2}\ell_{h,\lambda}^4
  +G^2Rh^{3/2}
  \bigr\}.
  \notag
\end{align}
Since $hL_f\le1$,
$L_f^2R^2h^3\le L_fR^2h^2$.  Taking square roots term by term yields
\begin{align}
  \sqrt m\limsup_{n\to\infty}\norm{Y_n}_{L^2}
  \le{}&C\bigl(R\sqrt{L_f}+L_f\sqrt d\bigr)h
  \label{eq:Y-final}\\
  &+C\bigl(G\sqrt R+K\sqrt G\,d^{1/4}
  \ell_{h,\lambda}^2\bigr)h^{3/4}.
  \notag
\end{align}
Because $Z_n=Y_n-\chi_h(X_{nh})$, stationarity and
\eqref{eq:chi-L2-input} further give
\begin{align}
  \sqrt m\limsup_{n\to\infty}\norm{Z_n}_{L^2}
  &\le\sqrt m\limsup_{n\to\infty}\norm{Y_n}_{L^2}
  +\sqrt m\norm{\chi_h}_{L^2(\pi_\lambda)}
  \notag\\
  &\le C\Bigl[R(\sqrt m+\sqrt{L_f})+L_f\sqrt d\Bigr]h
  \notag\\
  &\quad+C\Bigl[G\sqrt R+K\sqrt G\,d^{1/4}
  \ell_{h,\lambda}^2\Bigr]h^{3/4}.
  \label{eq:Z-final}
\end{align}

It remains to identify the limiting Euler law.  Define
$T_h(x):=x-h\nabla U_\lambda(x)$.  By~\eqref{eq:two-point-bounds} and
$hL_\lambda\le1$,
\begin{equation}\label{eq:Th-contraction}
  \norm{T_h(x)-T_h(y)}
  =\norm{(\Id-hA_\lambda(x,y))(x-y)}
  \le(1-mh)\norm{x-y}.
\end{equation}
A synchronous Gaussian coupling therefore gives
\begin{equation}\label{eq:kernel-contraction}
  \Wtwo(\mu Q_{\lambda,h},\nu Q_{\lambda,h})
  \le(1-mh)\Wtwo(\mu,\nu),
  \qquad \mu,\nu\in\Pcal_2(\R^d).
\end{equation}
The metric space $(\Pcal_2(\R^d),W_2)$ is complete, and
$Q_{\lambda,h}$ maps $\Pcal_2(\R^d)$ into itself.  The Banach fixed-point
theorem thus gives a unique invariant law $\widehat\pi_{\lambda,h}$ and
$\Law(\widehat X_n)\to\widehat\pi_{\lambda,h}$ in $W_2$.  Since
$X_{nh}\sim\pi_\lambda$, the coupling above and~\eqref{eq:Z-final} imply
\[
  \Wtwo(\pi_\lambda,\widehat\pi_{\lambda,h})
  \le\limsup_{n\to\infty}\norm{Z_n}_{L^2}.
\]
This proves~\eqref{eq:fixed-main}.
\end{proof}

\subsection{Proof of the end-to-end complexity theorem}

\begin{proof}[Proof of Theorem~\ref{thm:complexity}]
The universal Moreau bias of \citet[Proposition~5.9]{XinZhang2026} states
that
\begin{equation}\label{eq:moreau-bias}
  \sqrt m\,\Wtwo(\pi_\lambda,\pi)
  \le\frac{G^2\lambda}{4}.
\end{equation}
With $\lambda_\varepsilon=\varepsilon/G^2$,
\begin{equation}\label{eq:bias-quarter}
  \sqrt m\,\Wtwo(\pi_{\lambda_\varepsilon},\pi)
  \le\varepsilon/4.
\end{equation}
The first constraint in~\eqref{eq:h-epsilon} gives
$h_\varepsilon L_{\lambda_\varepsilon}\le c$.  Since
$L_{\lambda_\varepsilon}\ge\lambda_\varepsilon^{-1}$,
\[
  h_\varepsilon\le c\lambda_\varepsilon,
  \qquad
  \frac{\sqrt{h_\varepsilon}}
  {G\sqrt d\,\lambda_\varepsilon}
  \le\frac{\sqrt c}{\sqrt{d\varepsilon}}.
\]
Consequently,
\begin{equation}\label{eq:ell-control}
  \ell_{h_\varepsilon,\lambda_\varepsilon}
  \le C\ell_\varepsilon.
\end{equation}
Substituting the last two constraints in~\eqref{eq:h-epsilon} and \eqref{eq:ell-control} into Theorem~\ref{thm:fixed-lambda}, and reducing $c$ if necessary, yields
\begin{equation}\label{eq:disc-quarter}
  \sqrt m\,\Wtwo(\pi_{\lambda_\varepsilon},
  \widehat\pi_{\lambda_\varepsilon,h_\varepsilon})
  \le\varepsilon/4.
\end{equation}

By~\eqref{eq:kernel-contraction},
\begin{align}
  &\sqrt m\,\Wtwo(
  \mu_0Q_{\lambda_\varepsilon,h_\varepsilon}^{N},
  \widehat\pi_{\lambda_\varepsilon,h_\varepsilon})
  \notag\\
  &\quad\le
  \sqrt m\,e^{-mh_\varepsilon N}
  \Wtwo(\mu_0,
  \widehat\pi_{\lambda_\varepsilon,h_\varepsilon})
  \notag\\
  &\quad\le
  \sqrt m\,e^{-mh_\varepsilon N}
  \left\{
  \Wtwo(\mu_0,\pi)
  +\frac{\varepsilon}{2\sqrt m}
  \right\}
  \le\varepsilon/2,
  \label{eq:burnin}
\end{align}
where the second inequality uses the triangle inequality, \eqref{eq:bias-quarter}, and~\eqref{eq:disc-quarter}, and the last follows from~\eqref{eq:N-epsilon} with $C\ge2$. The triangle inequality together with~\eqref{eq:bias-quarter},~\eqref{eq:disc-quarter}, and~\eqref{eq:burnin} proves~\eqref{eq:end-to-end-error}.

Finally,~\eqref{eq:h-epsilon} implies
\begin{align*}
  h_\varepsilon^{-1}\le C\Biggl[
  L_f+\frac{G^2}{\varepsilon}
  +\frac{R(\sqrt m+\sqrt{L_f})+L_f\sqrt d}{\varepsilon}
  +\frac{
  (G\sqrt R+K\sqrt G\,d^{1/4}\ell_\varepsilon^2)^{4/3}}
  {\varepsilon^{4/3}}
  \Biggr].
\end{align*}
Using $(a+b)^{4/3}\le2^{1/3}(a^{4/3}+b^{4/3})$ proves \eqref{eq:N-explicit}.  For fixed model parameters and $\mu_0\in\Pcal_2(\R^d)$, $\Wtwo(\mu_0,\pi)$ is finite and independent of $\varepsilon$. Hence~\eqref{eq:N-explicit} gives~\eqref{eq:tilde-rate}.
\end{proof}

\section{Stationary Estimates and the Discrete Poisson Corrector}
\label{sec:corrector-proof}

\subsection{Stationary active trace and drift moments}

\begin{lemma}[Stationary drift and active-trace bounds]
\label{lem:stationary}
If $X\sim\pi_\lambda$, then
\begin{equation}\label{eq:stationary-estimates}
  \norm{\nabla U_\lambda}_{L^2(\pi_\lambda)}\le R,
  \qquad
  \E_{\pi_\lambda}a_\lambda\le GR,
  \qquad
  \E_{\pi_\lambda}\nabla U_\lambda=0.
\end{equation}
\end{lemma}

\begin{proof}
The weak integration-by-parts identities proved in
Proposition~\ref{prop:weighted-ibp} give
\begin{align}
  \E_{\pi_\lambda}\norm{\nabla U_\lambda}^2
  &=\E_{\pi_\lambda}\tr(\nabla^2f+H_\lambda)
  \le\tau_f+\E_{\pi_\lambda}a_\lambda,
  \label{eq:ibp-energy}\\
  \E_{\pi_\lambda}a_\lambda
  &=\E_{\pi_\lambda}
  \ip{\nabla g_\lambda}{\nabla U_\lambda}
  \le G\bigl(
  \E_{\pi_\lambda}\norm{\nabla U_\lambda}^2
  \bigr)^{1/2}.
  \label{eq:ibp-trace}
\end{align}
Let $r:=\norm{\nabla U_\lambda}_{L^2(\pi_\lambda)}$.  Then
$r^2\le\tau_f+Gr$, and therefore
\[
  r\le\frac{G+\sqrt{G^2+4\tau_f}}{2}=R.
\]
Substitution in~\eqref{eq:ibp-trace} gives
$\E_{\pi_\lambda}a_\lambda\le GR$.  The last identity in
\eqref{eq:stationary-estimates} is also contained in
Proposition~\ref{prop:weighted-ibp}.
\end{proof}

\begin{lemma}[All stationary drift moments]\label{lem:drift-Lp}
For every $p\ge2$,
\begin{equation}\label{eq:drift-Lp}
  \norm{\nabla U_\lambda}_{L^p(\pi_\lambda;\R^d)}
  \le K\sqrt p.
\end{equation}
\end{lemma}

\begin{proof}
Let $x_\star$ be the unique minimizer of $U_\lambda$.  Since
$\nabla U_\lambda(x_\star)=0$, the $L_f$-Lipschitz property of $\nabla f$
and~\eqref{eq:moreau-gradient-bound} imply
\begin{align}
  \norm{\nabla U_\lambda(x)}
  &\le\norm{\nabla f(x)-\nabla f(x_\star)}
  +\norm{\nabla g_\lambda(x)-\nabla g_\lambda(x_\star)}
  \notag\\
  &\le L_f\norm{x-x_\star}+2G.
  \label{eq:drift-pointwise}
\end{align}
Applying Proposition~\ref{prop:weighted-ibp} with
$V(x)=\tfrac12\norm{x-x_\star}^2$ and summing over the coordinates gives
\[
  d
  =\E_{\pi_\lambda}\ip{X-x_\star}{\nabla U_\lambda(X)}
  \ge m\E_{\pi_\lambda}\norm{X-x_\star}^2,
\]
where the inequality uses $\nabla U_\lambda(x_\star)=0$ and
$m$-strong convexity.  Hence
\[
  \E_{\pi_\lambda}\norm{X-x_\star}^2\le\frac dm.
\]
Moreover, the Bregman transport inequality and its concentration
consequence
\citep[Theorems~2.2.12 and~2.4.3(ii)]{Chewi2026LogConcaveSampling}
give
\[
  \Pp\left(
    \norm{X-x_\star}
    -\E\norm{X-x_\star}\ge t
  \right)
  \le e^{-mt^2/2},
  \qquad t\ge0.
\]
Tail integration therefore yields
\begin{equation}\label{eq:radial-moment}
  \norm{\norm{X-x_\star}}_{L^p}
  \le\frac{\sqrt d+2\sqrt p}{\sqrt m},
  \qquad p\ge2.
\end{equation}
Combining
\eqref{eq:drift-pointwise}--\eqref{eq:radial-moment} and using
$1\le\sqrt p$ gives
\[
  \norm{\nabla U_\lambda}_{L^p}
  \le2G+\frac{L_f}{\sqrt m}(\sqrt d+2\sqrt p)
  \le K\sqrt p.
\]
\end{proof}

\subsection{The reversible semigroup and the corrector operator}

For $u\in L^2(\pi_\lambda)$, let
\begin{equation}\label{eq:Pt-probabilistic}
  P_tu(x):=\E[u(X_t^x)],
\end{equation}
where $X^x$ denotes the Langevin diffusion~\eqref{eq:diffusion} started from
$x$.  Invariance of $\pi_\lambda$ and Jensen's inequality give
\begin{equation}\label{eq:Lp-contraction}
  \norm{P_tu}_{L^p(\pi_\lambda)}
  \le \norm u_{L^p(\pi_\lambda)},
  \qquad 1\le p\le\infty,
\end{equation}
and reversibility gives
$\ip{P_tu}{v}_{L^2}=\ip{u}{P_tv}_{L^2}$; see
\citet[Section~1.2]{Chewi2026LogConcaveSampling} for these Langevin
semigroup facts.  We use the sign convention
$\mathcal A_\lambda=-\mathcal L_\lambda$, where
$\mathcal L_\lambda$ is the usual Langevin generator.  On smooth compactly
supported functions,
\begin{equation}\label{eq:generator-on-core}
  \mathcal A_\lambda u
  =-\Delta u+\ip{\nabla U_\lambda}{\nabla u}.
\end{equation}
Its precise domain and the meaning of a bounded function of
$\mathcal A_\lambda$ are given in \Cref{app:operator-facts}.  This notation
is used below only to construct the corrector and state its moment bounds;
the main error recursion uses the Poisson identity~\eqref{eq:abstract-poisson}.
The same corrector construction, under stronger smoothness assumptions,
appears in \citet[Sections~2--2.1]{PedrottiWhalley2026}.

\subsection{Construction and moments of the corrector}

\begin{proposition}[The discrete Poisson corrector]
\label{prop:operator-corrector}
For $r\ge0$, define
\begin{equation}\label{eq:phi-h}
  \varphi_h(r):=
  \begin{cases}
  \displaystyle
  \frac{h-(1-e^{-hr})/r}{1-e^{-hr}},&r>0,\\[3mm]
  h/2,&r=0.
  \end{cases}
\end{equation}
Using the construction in \Cref{prop:operator-framework} coordinatewise, set
\begin{equation}\label{eq:chi-explicit}
  \chi_h:=\varphi_h(\mathcal A_\lambda)\nabla U_\lambda.
\end{equation}
Then $\chi_h$ is the unique mean-zero solution in
$L^2(\pi_\lambda;\R^d)$ of~\eqref{eq:abstract-poisson}, and
\begin{equation}\label{eq:phi-bounds}
  0\le\varphi_h(r)\le h,
  \qquad r\ge0.
\end{equation}
Consequently,
\begin{equation}\label{eq:chi-L2-proof}
  \norm{\chi_h}_{L^2(\pi_\lambda)}
  \le h\norm{\nabla U_\lambda}_{L^2(\pi_\lambda)}
  \le hR.
\end{equation}
\end{proposition}

\begin{proof}
Let
\[
  a_h(r):=1-e^{-hr},
  \qquad
  b_h(r):=\int_0^h(1-e^{-sr})\dd s.
\]
For every $r\ge0$,
\begin{equation}\label{eq:phi-identity}
  a_h(r)\varphi_h(r)=b_h(r).
\end{equation}
For $r>0$,
\[
  (1-e^{-hr})\varphi_h(r)
  =h-\frac{1-e^{-hr}}r
  =\int_0^h(1-e^{-sr})\dd s;
\]
At $r=0$, both sides of~\eqref{eq:phi-identity} are zero, and Taylor
expansion shows that $\varphi_h$ is continuous there.  The multiplication
and time-integration rules in
\Cref{prop:operator-framework} turn~\eqref{eq:phi-identity} into
\begin{equation}\label{eq:poisson-calculus-chain}
  (I-P_h)\chi_h
  =\int_0^h(I-P_s)\nabla U_\lambda\dd s.
\end{equation}
The same proposition preserves zero means, so Lemma~\ref{lem:stationary}
shows that $\chi_h$ has zero mean.

To prove~\eqref{eq:phi-bounds}, write $u=hr$.  The inequalities
$1-e^{-u}\le u$ and $e^u\ge1+u$ imply
\[
  0\le \frac1{1-e^{-u}}-\frac1u\le1.
\]
Thus $0\le\varphi_h(r)\le h$.  The norm bound
\eqref{eq:borel-calculus} and Lemma~\ref{lem:stationary} yield
\[
  \norm{\chi_h}_{L^2}
  \le h\norm{\nabla U_\lambda}_{L^2}
  \le hR.
\]

If $\widetilde\chi_h$ is another mean-zero solution, then
$v:=\chi_h-\widetilde\chi_h$ satisfies $P_hv=v$.  By~\eqref{eq:L2-decay},
\[
  \norm v_{L^2}
  =\norm{P_hv}_{L^2}
  \le e^{-mh}\norm v_{L^2},
\]
which forces $v=0$.  Hence the mean-zero solution is unique.
\end{proof}

Write
\begin{equation}\label{eq:psi-def}
  \varphi_h(r)=h\psi(hr),
  \qquad
  \psi(u):=
  \begin{cases}
    \displaystyle\frac1{1-e^{-u}}-\frac1u,&u>0,\\[2mm]
    1/2,&u=0.
  \end{cases}
\end{equation}

\begin{lemma}[$L^p$ bound for the corrector multiplier]
\label{lem:psi-Lp}
There is a universal constant $C$ such that, for every $p\ge2$, $h>0$,
and every scalar
$v\in L^2(\pi_\lambda)\cap L^p(\pi_\lambda)$ with
$\int v\dd\pi_\lambda=0$,
\begin{equation}\label{eq:psi-Lp}
  \norm{\psi(h\mathcal A_\lambda)v}_{L^p(\pi_\lambda)}
  \le Cp\norm v_{L^p(\pi_\lambda)}.
\end{equation}
The expression on the left is defined in \Cref{app:operator-facts};
the estimate says that this construction also has the stated $L^p$ bound.
\end{lemma}

\begin{proof}
The proof, including the precise statement of the external multiplier theorem
and the verification of all of its hypotheses, is given in
\Cref{prop:psi-Lp-appendix}.
\end{proof}

\begin{lemma}[Corrector moments]
\label{lem:corrector-Lp}
For every $p\ge2$,
\begin{equation}\label{eq:chi-Lp-proof}
  \norm{\chi_h}_{L^p(\pi_\lambda;\R^d)}
  \le ChKp^2.
\end{equation}
\end{lemma}

\begin{proof}
By~\eqref{eq:psi-def},
\[
  \chi_h=h\psi(h\mathcal A_\lambda)\nabla U_\lambda.
\]
Let $(\varepsilon_i)_{i=1}^d$ be independent Rademacher signs.  The
Khintchine inequalities, Lemma~\ref{lem:psi-Lp}, and Minkowski's inequality
give
\begin{align*}
  \norm{\psi(h\mathcal A_\lambda)\nabla U_\lambda}_{L^p(\ell_2)}
  &\le C\E_\varepsilon
  \norm{\psi(h\mathcal A_\lambda)
  \sum_{i=1}^d\varepsilon_i\partial_iU_\lambda}_{L^p}\\
  &\le Cp\E_\varepsilon
  \norm{\sum_{i=1}^d\varepsilon_i\partial_iU_\lambda}_{L^p}\\
  &\le Cp^{3/2}\norm{\nabla U_\lambda}_{L^p(\ell_2)}.
\end{align*}
Lemma~\ref{lem:drift-Lp} gives
$\norm{\nabla U_\lambda}_{L^p(\ell_2)}\le K\sqrt p$, and therefore
\[
  \norm{\chi_h}_{L^p(\ell_2)}
  \le ChKp^2.
\]
\end{proof}

Combining Proposition~\ref{prop:operator-corrector} and
Lemma~\ref{lem:corrector-Lp} gives
\eqref{eq:chi-L2-input}--\eqref{eq:chi-Lp-input}.

\subsection{Martingale decomposition of the local error}

The Markov property and~\eqref{eq:abstract-poisson} imply
\begin{align}
  \E[D_k\mid\mathcal F_{kh}]
  &=\int_0^h(I-P_s)\nabla U_\lambda(X_{kh})\dd s
  \notag\\
  &=(I-P_h)\chi_h(X_{kh}).
  \label{eq:D-conditional}
\end{align}
Define
\begin{equation}\label{eq:M-def}
  M_{k+1}:=D_k-
  \bigl\{\chi_h(X_{kh})-\chi_h(X_{(k+1)h})\bigr\}.
\end{equation}
Since
\[
  \E[\chi_h(X_{(k+1)h})\mid\mathcal F_{kh}]
  =P_h\chi_h(X_{kh}),
\]
Equation~\eqref{eq:D-conditional} proves
\eqref{eq:corrector-decomp}--\eqref{eq:M-centered}.  This is the
``telescoping increment plus martingale increment'' decomposition used by
\citet[Section~2.1, pp.~3--4]{PedrottiWhalley2026}.

\subsection{Short-time crossing and the martingale variance}

\begin{lemma}[A BV-type semigroup estimate]\label{lem:BV}
For every $t>0$,
\begin{align}
  \norm{P_t\nabla g_\lambda-\nabla g_\lambda}
  _{L^1(\pi_\lambda;\R^d)}
  &\le C\sqrt t\,\E_{\pi_\lambda}a_\lambda,
  \label{eq:BV-bound}\\
  \E\norm{\nabla g_\lambda(X_t)-\nabla g_\lambda(X_0)}^2
  &\le CG^2R\sqrt t.
  \label{eq:g-crossing}
\end{align}
\end{lemma}

\begin{proof}
Let $\phi:\R^d\to\R^d$ be bounded with $\norm\phi_\infty\le1$.
Reversibility and the time-integrated weak integration-by-parts formula in
Proposition~\ref{prop:weak-dirichlet-ibp} give
\begin{align}
  &\int\phi\cdot(P_t\nabla g_\lambda-\nabla g_\lambda)\dd\pi_\lambda
  \notag\\
  &\quad=\int\nabla g_\lambda\cdot(P_t\phi-\phi)\dd\pi_\lambda
  \notag\\
  &\quad=-\int_0^t\int H_\lambda:\nabla P_r\phi
  \dd\pi_\lambda\dd r.
  \label{eq:BV-duality}
\end{align}
For every unit vector $a\in\R^d$, the reverse local Poincar\'e inequality under $\mathrm{CD}(m,\infty)$ gives
\[
  \norm{\nabla P_r\ip{a}{\phi}}^2
  \le
  \frac{m}{e^{2mr}-1}
  \left\{
    P_r\bigl(\ip{a}{\phi}^2\bigr)
    -\bigl(P_r\ip{a}{\phi}\bigr)^2
  \right\}
  \le
  \frac{\norm{\phi}_\infty^2}{2r};
\]
see \citet[Theorem~4.7.2(iv), equation~(4.7.6), and the bounded-measurable extension discussed there]{BakryGentilLedoux2014}.
Since $P_r$ acts coordinatewise,
\[
  \nabla P_r\ip{a}{\phi}
  =(\nabla P_r\phi)^\top a.
\]
Taking the supremum over $\norm a=1$ yields
\begin{equation}\label{eq:reverse-Poincare}
  \norm{\nabla P_r\phi}_{\op,\infty}
  \le(2r)^{-1/2}\norm{\phi}_\infty.
\end{equation}
Since $H_\lambda\succeq0$,
\[
  \abs{H_\lambda:\nabla P_r\phi}
  \le a_\lambda\norm{\nabla P_r\phi}_{\op}.
\]
The integrable bound $r^{-1/2}$ justifies~\eqref{eq:BV-duality} by first
integrating over $[\varepsilon,t]$ and then letting $\varepsilon\downarrow0$.
Substitution in~\eqref{eq:BV-duality}, integration in $r$, and duality between
$L^1$ and $L^\infty$ prove~\eqref{eq:BV-bound}.

By stationarity and reversibility,
\begin{align*}
  &\E\norm{\nabla g_\lambda(X_t)-\nabla g_\lambda(X_0)}^2\\
  &\quad=2\ip{\nabla g_\lambda}
  {(I-P_t)\nabla g_\lambda}_{L^2(\pi_\lambda)}\\
  &\quad\le2G\norm{\nabla g_\lambda-P_t\nabla g_\lambda}_{L^1}
  \le CG\sqrt t\,\E_{\pi_\lambda}a_\lambda.
\end{align*}
Lemma~\ref{lem:stationary} completes the proof.
\end{proof}

Define the stationary drift modulus
\begin{equation}\label{eq:Gamma}
  \Gamma_\lambda(t)
  :=\E\norm{\nabla U_\lambda(X_t)-\nabla U_\lambda(X_0)}^2.
\end{equation}
From~\eqref{eq:diffusion}, Cauchy--Schwarz, stationarity, and
Lemma~\ref{lem:stationary},
\begin{align}
  \E\norm{X_t-X_0}^2
  &\le2t\int_0^t\E\norm{\nabla U_\lambda(X_s)}^2\dd s+4dt
  \notag\\
  &\le2R^2t^2+4dt.
  \label{eq:X-increment}
\end{align}
Combining this with the $L_f$-Lipschitz property of $\nabla f$ and
\eqref{eq:g-crossing} gives
\begin{align}
  \Gamma_\lambda(t)
  &\le2L_f^2\E\norm{X_t-X_0}^2
  +2\E\norm{\nabla g_\lambda(X_t)-\nabla g_\lambda(X_0)}^2
  \notag\\
  &\le C\left\{L_f^2(dt+R^2t^2)+G^2R\sqrt t\right\}.
  \label{eq:Gamma-bound}
\end{align}

Cauchy--Schwarz gives
\begin{equation}\label{eq:D-variance}
  \E\norm{D_k}^2
  \le h\int_0^h\Gamma_\lambda(s)\dd s.
\end{equation}
On the other hand, the corrector increment estimate proved at the end of
Appendix~\ref{app:bounded-generator-functions} gives
\begin{align}
  \E\norm{\chi_h(X_h)-\chi_h(X_0)}^2
  &\le h^2\E\norm{\nabla U_\lambda(X_h)-\nabla U_\lambda(X_0)}^2
  \notag\\
  &=h^2\Gamma_\lambda(h).
  \label{eq:chi-increment}
\end{align}
Using~\eqref{eq:M-def} and
$\norm{a+b}^2\le2\norm a^2+2\norm b^2$,
\begin{equation}\label{eq:M-prebound}
  \E\norm{M_{k+1}}^2
  \le2h\int_0^h\Gamma_\lambda(s)\dd s
  +2h^2\Gamma_\lambda(h).
\end{equation}
Substitution of~\eqref{eq:Gamma-bound}, together with
\[
  h\int_0^h s\dd s=\frac{h^3}{2},
  \qquad
  h\int_0^h s^2\dd s=\frac{h^4}{3},
  \qquad
  h\int_0^h\sqrt s\dd s=\frac{2h^{5/2}}{3},
\]
proves~\eqref{eq:M-var-input}.  Proposition~\ref{prop:corrector-input} is
therefore established.

\section{Shared-Noise Two-Point Curvature and Feedback}
\label{sec:shared-noise}

\subsection{A one-step shared-noise trace estimate}

\begin{proposition}[First moment of the shared-noise two-point curvature]
\label{prop:shared-trace}
Under Assumptions~\ref{ass:smooth}--\ref{ass:nonsmooth}, fix
$x,y\in\R^d$, and let
\begin{align}
  X_t&=x-\int_0^t\nabla U_\lambda(X_s)\dd s+\sqrt2B_t,
  \label{eq:local-exact}\\
  \widehat X_1&=y-h\nabla U_\lambda(y)+\sqrt2B_h.
  \label{eq:local-euler}
\end{align}
There exist universal constants $c,C>0$ such that, whenever
$hL_\lambda\le c$,
\begin{equation}\label{eq:shared-trace}
  \E\tr A_\lambda^g(X_h,\widehat X_1)
  \le CG\sqrt{\frac dh}.
\end{equation}
\end{proposition}

\begin{proof}
\textbf{Step 1: the smooth case.}
We first prove~\eqref{eq:shared-trace} under the additional assumption $g_\lambda\in C^2(\R^d)$.  The standing bounds then hold pointwise:
\[
  \norm{\nabla g_\lambda(x)}\le G,
  \qquad
  0\preceq\nabla^2g_\lambda(x)\preceq\lambda^{-1}\Id,
  \qquad
  m\Id\preceq\nabla^2U_\lambda(x)\preceq L_\lambda\Id,
  \quad x\in\R^d.
\]
Write
\begin{equation}\label{eq:bridge}
  W:=B_h\sim N(0,h\Id),
  \qquad
  \widetilde B_t:=B_t-\frac thW.
\end{equation}
The Gaussian endpoint $W$ and the Brownian bridge
$(\widetilde B_t)_{0\le t\le h}$ are independent.  Conditional on the entire
bridge, both $X_t$ and $\widehat X_1$ are deterministic functions of the same
vector $W$.  For $0\le\theta\le1$, define
\begin{equation}\label{eq:T-theta}
  T_\theta(W):=(1-\theta)\widehat X_1(W)+\theta X_h(W).
\end{equation}

Let $K_t:=D_WX_t$.  Differentiating~\eqref{eq:local-exact} with respect to
$W$ gives
\begin{equation}\label{eq:K-ODE}
  \dot K_t=-\nabla^2U_\lambda(X_t)K_t+\frac{\sqrt2}{h}\Id,
  \qquad K_0=0.
\end{equation}
Integrating~\eqref{eq:K-ODE} gives
\[
  K_t
  =
  \frac{\sqrt2\,t}{h}\Id
  -\int_0^t\nabla^2U_\lambda(X_s)K_s\dd s.
\]
Since
$\norm{\nabla^2U_\lambda(X_s)}_{\op}\le L_\lambda$,
\[
  \norm{K_t}_{\op}
  \le
  \frac{\sqrt2\,t}{h}
  +L_\lambda\int_0^t\norm{K_s}_{\op}\dd s.
\]
Gronwall's inequality therefore yields
\[
  \norm{K_t}_{\op}
  \le
  \frac{\sqrt2\,t}{h}e^{L_\lambda t}
  \le
  \sqrt2e^{hL_\lambda},
  \qquad 0\le t\le h.
\]
Moreover, evaluating the integral equation at $t=h$ gives
\[
  K_h-\sqrt2\Id
  =
  -\int_0^h\nabla^2U_\lambda(X_s)K_s\dd s,
\]
and hence
\[
  \norm{K_h-\sqrt2\Id}_{\op}
  \le
  hL_\lambda\sup_{0\le s\le h}\norm{K_s}_{\op}
  \le
  \sqrt2hL_\lambda e^{hL_\lambda}.
\]
Thus
\begin{equation}\label{eq:K-estimates}
  \sup_{0\le t\le h}\norm{K_t}_{\op}
  \le\sqrt2e^{hL_\lambda},
  \qquad
  \norm{K_h-\sqrt2\Id}_{\op}
  \le\sqrt2hL_\lambda e^{hL_\lambda}.
\end{equation}
Choose $c$ so that $ce^c\le1/4$.  Since
$D_W\widehat X_1=\sqrt2\Id$,
\begin{equation}\label{eq:J-lower}
  J_\theta:=D_WT_\theta
  =(1-\theta)\sqrt2\Id+\theta K_h,
  \qquad
  \sym J_\theta\succeq2^{-1/2}\Id.
\end{equation}

Define the Hessian averaged along the segment joining the two endpoints by
\begin{equation}\label{eq:H-chord}
  \overline H_\lambda(X_h,\widehat X_1)
  :=\int_0^1\nabla^2g_\lambda(T_\theta(W))\dd\theta.
\end{equation}
Then
\[
  \overline H_\lambda(X_h,\widehat X_1)(X_h-\widehat X_1)
  =\nabla g_\lambda(X_h)-\nabla g_\lambda(\widehat X_1).
\]
The positive-semidefinite Cauchy--Schwarz inequality gives
\begin{equation}\label{eq:Ag-below-H}
  A_\lambda^g(X_h,\widehat X_1)
  \preceq\overline H_\lambda(X_h,\widehat X_1).
\end{equation}
Indeed, for every $v\in\R^d$, the quadratic form on the left is
\[
  \frac{
  \bigl(v^\top\overline H_\lambda(X_h-\widehat X_1)\bigr)^2}
  {(X_h-\widehat X_1)^\top
  \overline H_\lambda(X_h-\widehat X_1)}
  \le v^\top\overline H_\lambda v.
\]
Consequently,
\begin{equation}\label{eq:trace-below-H}
  \tr A_\lambda^g(X_h,\widehat X_1)
  \le\int_0^1\tr\nabla^2g_\lambda(T_\theta(W))\dd\theta.
\end{equation}

Fix the bridge and $\theta$.  By the chain rule, symmetry of the Hessian,
and~\eqref{eq:J-lower},
\begin{align}
  \operatorname{div}_W\{\nabla g_\lambda(T_\theta(W))\}
  &=\tr\{\nabla^2g_\lambda(T_\theta(W))J_\theta(W)\}
  \notag\\
  &=\tr\{\nabla^2g_\lambda(T_\theta(W))\sym J_\theta(W)\}
  \notag\\
  &\ge2^{-1/2}\tr\nabla^2g_\lambda(T_\theta(W)).
  \label{eq:div-lower}
\end{align}
Gaussian integration by parts for $W\sim N(0,h\Id)$ yields
\begin{align}
  \E_W\operatorname{div}_W\{\nabla g_\lambda(T_\theta(W))\}
  &=\frac1h\E_W\ip{W}{\nabla g_\lambda(T_\theta(W))}
  \notag\\
  &\le\frac Gh\E\norm W
  \le G\sqrt{\frac dh}.
  \label{eq:Gaussian-IBP}
\end{align}
Combining~\eqref{eq:trace-below-H}--\eqref{eq:Gaussian-IBP} and integrating
over $\theta$ and the Brownian bridge proves~\eqref{eq:shared-trace} in the
smooth case.

\medskip
\noindent\textbf{Step 2: passage to the $C^{1,1}$ Moreau envelope.}
Let $\rho_\delta$ be a standard smooth mollifier and, only within this proof,
set
\begin{equation}\label{eq:mollifier}
  g_{\lambda,\delta}:=g_\lambda*\rho_\delta,
  \qquad
  U_{\lambda,\delta}:=f+g_{\lambda,\delta}.
\end{equation}
Mollification preserves the uniform bounds
\begin{equation}\label{eq:mollifier-bounds}
  \sup_x\norm{\nabla g_{\lambda,\delta}(x)}\le G,
  \qquad
  0\preceq\nabla^2g_{\lambda,\delta}\preceq\lambda^{-1}\Id,
  \qquad
  \norm{\nabla g_{\lambda,\delta}-\nabla g_\lambda}_\infty\to0.
\end{equation}
Let $X^{(\delta)}$ and $\widehat X_1^{(\delta)}$ be the exact and Euler
endpoints associated with $U_{\lambda,\delta}$, driven by the same Brownian
path as the original endpoints.  If
\[
  \eta_\delta
  :=\norm{\nabla U_{\lambda,\delta}-\nabla U_\lambda}_\infty,
\]
then Gronwall's inequality gives
\begin{equation}\label{eq:endpoint-convergence}
  \sup_{0\le t\le h}\norm{X_t^{(\delta)}-X_t}
  \le h\eta_\delta e^{hL_\lambda}\longrightarrow0,
  \qquad
  \norm{\widehat X_1^{(\delta)}-\widehat X_1}
  \le h\eta_\delta\longrightarrow0.
\end{equation}

For the limiting argument, define
\begin{equation}\label{eq:sigma-lsc}
  \Sigma(u,v):=
  \begin{cases}
  \norm v^2/\ip{u}{v},&v\ne0,\\
  0,&v=0.
  \end{cases}
\end{equation}
On the closed set
$\{(u,v):\ip{u}{v}\ge\lambda\norm v^2\}$, the function $\Sigma$ is lower
semicontinuous.  The trace of the matrix in~\eqref{eq:Ag-def} is precisely
$\Sigma(x-y,\nabla g_\lambda(x)-\nabla g_\lambda(y))$.  Therefore
\eqref{eq:mollifier-bounds}--\eqref{eq:endpoint-convergence}, cocoercivity,
and Fatou's lemma imply
\begin{align}
  \E\tr A_\lambda^g(X_h,\widehat X_1)
  &\le\liminf_{\delta\downarrow0}
  \E\tr A_{\lambda,\delta}^g
  (X_h^{(\delta)},\widehat X_1^{(\delta)})
  \notag\\
  &\le CG\sqrt{\frac dh}.
  \label{eq:Fatou-shared}
\end{align}
Thus~\eqref{eq:shared-trace} holds under Assumptions~\ref{ass:smooth}--\ref{ass:nonsmooth}, completing the proof.
\end{proof}

\subsection{From the one-step trace estimate to curvature feedback}
\label{sec:feedback-proof}

In the global synchronous coupling, condition on the starting points of the
$(k-1)$st step and apply Proposition~\ref{prop:shared-trace}.  For $k\ge1$,
\begin{equation}\label{eq:S-bounds}
  S_k:=\tr A_k^g,
  \qquad
  0\le S_k\le\lambda^{-1},
  \qquad
  \E[S_k\mid\mathcal F_{(k-1)h}]
  \le CG\sqrt{\frac dh}.
\end{equation}
Since $X_0=\widehat X_0$, one has $S_0=0$.  Hence, for every $k\ge0$,
\begin{equation}\label{eq:S-unconditional}
  \E S_k\le CG\sqrt{d/h}.
\end{equation}
No independence between $S_k$ and $\chi_h(X_{kh})$ is assumed.

\begin{lemma}[A Holder product bound]\label{lem:Holder-product}
Let a nonnegative random variable $S$ and a random vector $C$ satisfy
\begin{equation}\label{eq:Holder-assumption}
  0\le S\le L,
  \qquad
  \E S\le B.
\end{equation}
Set $B_0:=\min\{B,L\}$.  If $B_0=0$, then $S=0$ almost surely.  If
$B_0>0$, then, for every $p>1$,
\begin{equation}\label{eq:Holder-general}
  \E[\norm C^2S]
  \le\norm C_{L^{2p}}^2 B_0^{1-1/p}L^{1/p}.
\end{equation}
If, in addition, $\norm C_{L^r}\le Hr^a$ for every $r\ge2$, then
\begin{equation}\label{eq:Holder-optimized}
  \E[\norm C^2S]
  \le C_aH^2B_0
  \left\{1+\log\!\left(e+\frac{L}{B_0}\right)\right\}^{2a}.
\end{equation}
\end{lemma}

\begin{proof}
Let $q=p/(p-1)$.  Holder's inequality and
$S^q\le L^{q-1}S$ give
\begin{align*}
  \E[\norm C^2S]
  &\le\norm C_{L^{2p}}^2\norm S_{L^q}
  \le\norm C_{L^{2p}}^2(L^{q-1}B_0)^{1/q}\\
  &=\norm C_{L^{2p}}^2B_0^{1-1/p}L^{1/p},
\end{align*}
which proves~\eqref{eq:Holder-general}.  Choose
$p:=2+\log(L/B_0)$.  Then
\[
  B_0^{1-1/p}L^{1/p}
  =B_0(L/B_0)^{1/p}\le eB_0,
\]
and
$p\le C\{1+\log(e+L/B_0)\}$.  Substitution of
$\norm C_{L^{2p}}\le H(2p)^a$ proves~\eqref{eq:Holder-optimized}.
\end{proof}

We now apply Lemma~\ref{lem:Holder-product} with
\begin{equation}\label{eq:Holder-application}
  C=\chi_h(X_{kh}),
  \qquad
  S=S_k,
  \qquad
  L=\lambda^{-1},
  \qquad
  B=CG\sqrt{d/h}.
\end{equation}
By~\eqref{eq:chi-Lp-input}, one may take $H=ChK$ and $a=2$.  Since
$A_k^g\succeq0$,
\begin{equation}\label{eq:quadratic-trace}
  \chi_h(X_{kh})^\top A_k^g\chi_h(X_{kh})
  \le\norm{\chi_h(X_{kh})}^2S_k.
\end{equation}
If $B_0=B$, then
$L/B_0\le C\sqrt h/(G\sqrt d\,\lambda)$; if $B_0=L$, then $L/B_0=1$.
Thus
\begin{equation}\label{eq:log-control}
  B_0\le CG\sqrt{d/h},
  \qquad
  1+\log\!\left(e+\frac{L}{B_0}\right)
  \le C\ell_{h,\lambda}.
\end{equation}
Equations~\eqref{eq:Holder-optimized}--\eqref{eq:log-control} yield
\begin{align*}
  \E\left[
  \chi_h(X_{kh})^\top A_k^g\chi_h(X_{kh})
  \right]
  &\le Ch^2K^2G\sqrt{d/h}\,\ell_{h,\lambda}^4\\
  &=CK^2G\sqrt d\,h^{3/2}\ell_{h,\lambda}^4.
\end{align*}
This proves Proposition~\ref{prop:feedback-input}.

\section{Conclusion}\label{sec:conclusion}

We established Wasserstein error bounds for classical MYULA with a smooth strongly convex component and a convex Lipschitz nonsmooth term.
A discrete Poisson corrector, combined with active-trace and shared-noise curvature estimates, yields an $O(h)+\widetilde O(h^{3/4})$ invariant-measure bias bound. Under the stated step-size restriction, the error coefficients depend only logarithmically on the inverse smoothing parameter. Combining this estimate with the Moreau approximation bias and Wasserstein contraction gives a sufficient $\widetilde O(\varepsilon^{-4/3})$ iteration complexity for $\sqrt m\,\Wtwo(\mu_N,\pi)\le\varepsilon$, with model parameters and initialization fixed.

\appendix

\section{Weak Hessians and Integration by Parts}
\label{app:weak-hessian}

We record the facts used in Lemmas~\ref{lem:stationary} and~\ref{lem:BV}; see \citet[Appendices~A--B]{XinZhang2026}, \citet{EvansGariepy2015}, and \citet{Leoni2017}.

\subsection{Weak derivatives and weak Hessians}

\begin{definition}[Weak derivative]\label{def:weak-derivative}
Let $u,v\in L^1_{\mathrm{loc}}(\R^d)$.  We say that $v$ is the weak partial
derivative $\partial_i u$ if
\begin{equation}\label{eq:weak-derivative}
  \int_{\R^d}u(x)\,\partial_i\phi(x)\dd x
  =-\int_{\R^d}v(x)\phi(x)\dd x
  \qquad\text{for every }\phi\in C_c^\infty(\R^d).
\end{equation}
The function $v$, if it exists, is unique up to equality on a Lebesgue-null
set.
\end{definition}

A Lipschitz function is differentiable almost everywhere by Rademacher's theorem.  Moreover, its almost-everywhere derivative is its weak derivative; see \citet[Sections~3.1--3.2]{EvansGariepy2015}.  Since $g_\lambda\in C^{1,1}$, each component of $\nabla g_\lambda$ is Lipschitz and therefore has a bounded weak gradient.

\begin{definition}[Weak Hessian]\label{def:weak-hessian}
Let $u\in C^1(\R^d)$ and assume that $\nabla u$ is locally Lipschitz.  The
weak Hessian $H_u=\nabla^2u$ is the matrix field whose $(i,j)$ entry is the
weak derivative $\partial_j(\partial_i u)$.  It agrees almost everywhere with
the classical derivative of $\nabla u$ supplied by Rademacher's theorem.
\end{definition}

Different measurable representatives of $H_u$ can differ on a null set, but
all integrals against an absolutely continuous probability law are the same.
This is why the choice of representative never affects the estimates in the
paper.

\begin{proposition}[Basic properties of a weak Hessian]
\label{prop:weak-hessian-properties}
Let $u\in C^{1,1}(\R^d)$ be convex and suppose
$\operatorname{Lip}(\nabla u)\le L$.  Then its weak Hessian $H_u$ can be
chosen measurable and satisfies
\begin{equation}\label{eq:weak-H-general-bounds}
  H_u(x)=H_u(x)^\top,
  \qquad
  0\preceq H_u(x)\preceq L\Id
  \quad\text{for a.e. }x.
\end{equation}
If $u=g_\lambda$, then $L=\lambda^{-1}$ and
$a_\lambda=\tr H_\lambda$ satisfies~\eqref{eq:weak-hessian-bounds}.
\end{proposition}

\begin{proof}
Rademacher's theorem gives the almost-everywhere derivative of the Lipschitz
map $\nabla u$.  Equality of mixed weak derivatives implies symmetry almost
everywhere.  A convenient way to see the matrix inequalities is to mollify.
Let $u_\delta=u*\rho_\delta$, where $\rho_\delta$ is a standard nonnegative
smooth mollifier.  Convexity is preserved by convolution, so
$\nabla^2u_\delta\succeq0$.  Also
$\operatorname{Lip}(\nabla u_\delta)\le L$, hence
$\nabla^2u_\delta\preceq L\Id$.  The gradients
$\nabla u_\delta$ converge locally uniformly to $\nabla u$, and their
classical derivatives converge to $H_u$ in the sense of distributions and,
after passing to a subsequence, weakly-* in $L^\infty_{\mathrm{loc}}$.
Testing the inequalities against nonnegative scalar test functions and fixed
vectors shows that the same inequalities hold for $H_u$ almost everywhere.
\end{proof}

\subsection{Weighted integration by parts under the Gibbs law}

Let $x_\star$ be the minimizer of $U_\lambda$.  Strong convexity gives
\begin{equation}\label{eq:Gaussian-tail-from-convexity}
  U_\lambda(x)
  \ge U_\lambda(x_\star)+\frac m2\norm{x-x_\star}^2.
\end{equation}
Hence $\pi_\lambda$ has Gaussian tails.  In particular, every polynomial in
$\norm x$ is integrable.  Since $\nabla U_\lambda$ is globally Lipschitz and
vanishes at $x_\star$, it grows at most linearly.  These two observations
justify the cutoff limits below.

Choose $\eta\in C_c^\infty(\R^d)$ such that
$0\le\eta\le1$, $\eta=1$ on the unit ball, and $\eta=0$ outside the ball of
radius $2$.  Put $\eta_R(x)=\eta(x/R)$.  Then
\begin{equation}\label{eq:cutoff-properties}
  \eta_R\uparrow1,
  \qquad
  \norm{\nabla\eta_R}_\infty\le C/R,
  \qquad
  \operatorname{supp}(\nabla\eta_R)
  \subset\{R\le\norm x\le2R\}.
\end{equation}

\begin{proposition}[Weighted weak integration by parts]
\label{prop:weighted-ibp}
Let $V\in C^{1,1}(\R^d)$, and assume that $\nabla V$ grows at most linearly.
Let $H_V$ denote its weak Hessian.  If the terms below are integrable, then
for every $1\le i,j\le d$,
\begin{equation}\label{eq:weighted-ibp-general}
  \int (H_V)_{ij}\dd\pi_\lambda
  =\int \partial_iV\,\partial_jU_\lambda\dd\pi_\lambda.
\end{equation}
Moreover,
\begin{equation}\label{eq:mean-drift-zero-appendix}
  \int\partial_iU_\lambda\dd\pi_\lambda=0.
\end{equation}
In particular,
\begin{align}
  \int\norm{\nabla U_\lambda}^2\dd\pi_\lambda
  &=\int\tr(\nabla^2f+H_\lambda)\dd\pi_\lambda,
  \label{eq:weighted-energy-appendix}\\
  \int a_\lambda\dd\pi_\lambda
  &=\int\ip{\nabla g_\lambda}{\nabla U_\lambda}
  \dd\pi_\lambda.
  \label{eq:weighted-trace-appendix}
\end{align}
\end{proposition}

\begin{proof}
We first work with the unnormalized density $e^{-U_\lambda}$.  Apply the weak
derivative identity~\eqref{eq:weak-derivative} to the function
$\partial_iV$ and the compactly supported test function
$\eta_Re^{-U_\lambda}$.  This is legitimate because $U_\lambda\in C^1$ and
$\eta_Re^{-U_\lambda}$ can be approximated in the relevant Sobolev norm by
smooth compactly supported test functions.  We obtain
\begin{align}
  \int (H_V)_{ij}\eta_Re^{-U_\lambda}\dd x
  &=-\int\partial_iV\,
  \partial_j(\eta_Re^{-U_\lambda})\dd x
  \notag\\
  &=\int\eta_R\partial_iV\,\partial_jU_\lambda
  e^{-U_\lambda}\dd x
  -\int\partial_iV\,\partial_j\eta_R
  e^{-U_\lambda}\dd x.
  \label{eq:cutoff-ibp}
\end{align}
The last integral is supported on the annulus in~\eqref{eq:cutoff-properties}.
Its absolute value is bounded by
\[
  \frac CR\int_{R\le\norm x\le2R}
  \abs{\partial_iV(x)}e^{-U_\lambda(x)}\dd x,
\]
which tends to zero by the at-most-linear growth of $\nabla V$ and the
Gaussian tail~\eqref{eq:Gaussian-tail-from-convexity}.  Dominated convergence
in the first two terms of~\eqref{eq:cutoff-ibp}, followed by normalization,
proves~\eqref{eq:weighted-ibp-general}.

For~\eqref{eq:mean-drift-zero-appendix}, use
$\partial_i(e^{-U_\lambda})=-\partial_iU_\lambda e^{-U_\lambda}$:
\begin{align*}
  \int\eta_R\partial_iU_\lambda e^{-U_\lambda}\dd x
  &=-\int\eta_R\partial_i(e^{-U_\lambda})\dd x
  =\int\partial_i\eta_R e^{-U_\lambda}\dd x.
\end{align*}
The right-hand side tends to zero by~\eqref{eq:cutoff-properties} and the
Gaussian tail.  Finally, take $V=U_\lambda$ in
\eqref{eq:weighted-ibp-general}, set $i=j$, and sum over $i$ to obtain
\eqref{eq:weighted-energy-appendix}.  Taking $V=g_\lambda$ gives
\eqref{eq:weighted-trace-appendix}.
\end{proof}

\subsection{The weak integration by parts used in the BV estimate}

The proof of Lemma~\ref{lem:BV} uses one further identity.  We state it in a
form tailored to the present semigroup, without introducing a closed
Dirichlet form or a Sobolev form domain.

\begin{proposition}[Weak generator integration by parts]
\label{prop:weak-dirichlet-ibp}
Let $\phi=(\phi_1,\ldots,\phi_d)$ be bounded and let $r>0$.  Then
\begin{equation}\label{eq:weak-dirichlet-ibp}
  \int\nabla g_\lambda\cdot
  \bigl(-\mathcal A_\lambda P_r\phi\bigr)\dd\pi_\lambda
  =-\int H_\lambda:\nabla P_r\phi\dd\pi_\lambda.
\end{equation}
Here
$H_\lambda:\nabla P_r\phi
=\sum_{i,j}(H_\lambda)_{ij}\partial_j(P_r\phi_i)$.
\end{proposition}

\begin{proof}
Fix a coordinate $i$, and write
$u_i:=\partial_i g_\lambda$ and $v_i:=P_r\phi_i$.  The function $u_i$ is
bounded, and its weak gradient is the $i$th row of $H_\lambda$.  For $r>0$, semigroup regularization gives a bounded gradient for $v_i$.
Proposition~\ref{prop:operator-framework} also gives
$v_i=P_r\phi_i\in D(\mathcal A_\lambda)$ and
\[
  \norm{\mathcal A_\lambda v_i}_{L^2}
  \le\frac1{er}\norm{\phi_i}_{L^2}.
\]

By self-adjointness of $\mathcal A_\lambda$ and
\eqref{eq:generator-on-core}, for every $\zeta\in C_c^\infty(\R^d)$,
\[
  \int \zeta\,\mathcal A_\lambda v_i\dd\pi_\lambda
  =\int v_i\,\mathcal A_\lambda\zeta\dd\pi_\lambda
  =\int\ip{\nabla v_i}{\nabla\zeta}\dd\pi_\lambda.
\]
The last equality uses
$\mathcal A_\lambda\zeta
=-e^{U_\lambda}\operatorname{div}(e^{-U_\lambda}\nabla\zeta)$
and only the weak first derivatives of $v_i$.

Let $\eta_R$ be the cutoff from~\eqref{eq:cutoff-properties}.
The compactly supported Lipschitz function $u_i\eta_R$ admits smooth
compactly supported approximations whose functions and weak gradients
converge in $L^2(\pi_\lambda)$: this follows by mollification on a fixed
compact set, where the Gibbs density is bounded above and below by
positive constants.  Since $\mathcal A_\lambda v_i\in L^2(\pi_\lambda)$
and $\nabla v_i$ is bounded, the preceding identity therefore extends to
$\zeta=u_i\eta_R$.  Using $\partial_j u_i=(H_\lambda)_{ij}$ gives
\begin{align}
  \int \eta_R u_i(-\mathcal A_\lambda v_i)\dd\pi_\lambda
  &=-\sum_{j=1}^d\int
  \eta_R(H_\lambda)_{ij}\partial_jv_i\dd\pi_\lambda
  \notag\\
  &\quad-
  \int u_i\ip{\nabla\eta_R}{\nabla v_i}\dd\pi_\lambda.
  \label{eq:weak-generator-cutoff}
\end{align}

The last term in~\eqref{eq:weak-generator-cutoff} tends to zero because $\abs{u_i}\le G$, $\norm{\nabla v_i}_\infty<\infty$, and $\norm{\nabla\eta_R}_\infty\le C/R$.  The other two terms converge by dominated convergence, using $\mathcal A_\lambda v_i\in L^2$ and $\norm{H_\lambda}_{\op}\le\lambda^{-1}$.  Thus
\[
  \int u_i(-\mathcal A_\lambda v_i)\dd\pi_\lambda
  =-\sum_{j=1}^d\int
  (H_\lambda)_{ij}\partial_jv_i\dd\pi_\lambda.
\]
Summing over $i$ proves~\eqref{eq:weak-dirichlet-ibp}.
\end{proof}

\section{Operator Facts Used by the Corrector}
\label{app:operator-facts}

We collect the semigroup and spectral-calculus facts used by the corrector; see \citet[Appendices~A.1--A.2 and~A.4]{BakryGentilLedoux2014}. The $L^p$ multiplier theorem is stated in \Cref{prop:CD-multiplier}.

\subsection{From the Markov semigroup to a nonnegative generator}
\label{app:langevin-generator}

Fix $\lambda>0$ under Assumptions~\ref{ass:smooth}--\ref{ass:nonsmooth}.
Throughout this appendix, $(P_t)$ is the transition semigroup
\eqref{eq:Pt-probabilistic} of the exact Langevin diffusion
\eqref{eq:diffusion}, not the Euler kernel $Q_{\lambda,h}$.  In particular,
$U_\lambda\in C^{1,1}$ is $m$-strongly convex and
$\nabla U_\lambda$ is globally Lipschitz.  Set
\[
  \mathsf H:=L^2(\pi_\lambda),
  \qquad
  \ip{u}{v}_{\mathsf H}:=\int uv\dd\pi_\lambda.
\]
We use real-valued functions, identified up to $\pi_\lambda$-null sets.
The operator $\mathcal A_\lambda$ below acts on functions and is distinct from
the two-point matrix $A_\lambda(x,y)$ in~\eqref{eq:A-def}.

For a contraction semigroup, \emph{strong continuity} means
$\norm{P_tu-u}_{\mathsf H}\to0$ as $t\downarrow0$ for every
$u\in\mathsf H$.  This is the continuity condition in
\citet[Definition~A.1.1(iii)]{BakryGentilLedoux2014}; the semigroup property
then gives continuity of $t\mapsto P_tu$ at every $t\ge0$.
A bounded operator $T$ is \emph{self-adjoint} if
$\ip{Tu}{v}_{\mathsf H}=\ip{u}{Tv}_{\mathsf H}$ for all $u,v\in\mathsf H$.
Thus reversibility is exactly self-adjointness of $P_t$ on $L^2(\pi_\lambda)$.

The generator is defined on its full $L^2$ domain by
\begin{equation}\label{eq:generator-from-semigroup}
  D(\mathcal A_\lambda)
  :=\left\{u\in\mathsf H:
  \lim_{t\downarrow0}\frac{u-P_tu}{t}
  \text{ exists in }\mathsf H\right\},
  \qquad
  \mathcal A_\lambda u
  :=\lim_{t\downarrow0}\frac{u-P_tu}{t}.
\end{equation}
The minus sign relative to the usual SDE generator is intentional:
$\mathcal A_\lambda=-\mathcal L_\lambda$.
An operator ``on $\mathsf H$'' need not have domain equal to $\mathsf H$.
For a densely defined operator $A$, the adjoint is given by $A^*v=w$ when
$\ip{Au}{v}_{\mathsf H}=\ip{u}{w}_{\mathsf H}$ for every $u\in D(A)$;
$D(A^*)$ consists of the $v$ for which such a $w\in\mathsf H$ exists.
Self-adjointness means $D(A^*)=D(A)$ and $A^*=A$, not just symmetry on
$D(A)$.  Nonnegativity means $\ip{Au}{u}_{\mathsf H}\ge0$ on $D(A)$.
See \citet[Appendix~A.2]{BakryGentilLedoux2014}.

\begin{proposition}[Generator associated with the reversible diffusion]
\label{prop:generator-framework}
The operators $(P_t)_{t\ge0}$ form a strongly continuous self-adjoint
contraction semigroup on $\mathsf H$.  The operator
$\mathcal A_\lambda$ in~\eqref{eq:generator-from-semigroup} has dense domain,
is nonnegative and self-adjoint, and satisfies
\begin{equation}\label{eq:semigroup-from-A}
  P_t=e^{-t\mathcal A_\lambda},
  \qquad t\ge0,
\end{equation}
where the exponential is defined by the spectral calculus described below.
Moreover, $1\in D(\mathcal A_\lambda)$, $\mathcal A_\lambda1=0$, and
$C_c^\infty(\R^d)\subset D(\mathcal A_\lambda)$, with
$\mathcal A_\lambda u=-\Delta u+\ip{\nabla U_\lambda}{\nabla u}$ there.
\end{proposition}

\begin{proof}
The Markov property gives $P_{s+t}=P_sP_t$ and $P_0=I$.
Invariance and Jensen's inequality give~\eqref{eq:Lp-contraction}, and
reversibility gives self-adjointness.  To check strong continuity, use
$X_0\sim\pi_\lambda$.  For bounded continuous $u$, conditional Jensen gives
\[
  \norm{P_tu-u}_{\mathsf H}^2
  \le\E\abs{u(X_t)-u(X_0)}^2\longrightarrow0
\]
by path continuity and dominated convergence.  Density of bounded
continuous functions in $\mathsf H$ and contraction extend the limit to
every $u\in\mathsf H$.

The symmetric-semigroup theorem in
\citet[Appendices~A.1--A.2]{BakryGentilLedoux2014} now gives dense domain
and self-adjointness of the full generator.  Nonnegativity follows from
$\ip{u-P_tu}{u}_{\mathsf H}\ge0$ and~\eqref{eq:generator-from-semigroup}.
The spectral semigroup in
\citet[Appendix~A.4, immediately after Theorem~A.4.2]{BakryGentilLedoux2014}
has generator $-\mathcal A_\lambda$; uniqueness of the semigroup determined
by its full generator gives~\eqref{eq:semigroup-from-A}.

Since $P_t1=1$, the generator annihilates constants.  For
$u\in C_c^\infty(\R^d)$, put
$\mathcal L_\lambda u=\Delta u-\ip{\nabla U_\lambda}{\nabla u}\in\mathsf H$.
It\^o's formula gives
\[
  \frac{P_tu-u}{t}
  =\frac1t\int_0^tP_s\mathcal L_\lambda u\dd s
  \longrightarrow\mathcal L_\lambda u
  \quad\text{in }\mathsf H,
\]
Strong continuity gives the last limit and hence both assertions on
$C_c^\infty$.  Self-adjointness concerns the full
domain~\eqref{eq:generator-from-semigroup}.
\end{proof}

\subsection{Bounded functions of the generator}
\label{app:bounded-generator-functions}

We use the spectral calculus for $\mathcal A_\lambda$ described in \citet[Appendix~A.4]{BakryGentilLedoux2014}. For bounded continuous $q$, the operator $q(\mathcal A_\lambda)$ is defined on all of $\mathsf H$.

\begin{proposition}[Functional-calculus rules used in the paper]
\label{prop:operator-framework}
For every bounded continuous function $q:[0,\infty)\to\R$, the spectral
calculus defines a bounded self-adjoint operator
$q(\mathcal A_\lambda)$ on all of $\mathsf H$.  For bounded continuous
$q,q_1,q_2$ and $\alpha,\beta\in\R$,
\begin{align}
  \norm{q(\mathcal A_\lambda)}_{L^2\to L^2}
  &\le \sup_{r\ge0}\abs{q(r)},
  \label{eq:borel-calculus}\\
  (\alpha q_1+\beta q_2)(\mathcal A_\lambda)
  &=\alpha q_1(\mathcal A_\lambda)
    +\beta q_2(\mathcal A_\lambda),
  \label{eq:calculus-linearity}\\
  q_1(\mathcal A_\lambda)q_2(\mathcal A_\lambda)
  &=(q_1q_2)(\mathcal A_\lambda).
  \label{eq:calculus-product}
\end{align}
If $q_1\le q_2$ pointwise, then
\begin{equation}\label{eq:calculus-order}
  \ip{u}{q_1(\mathcal A_\lambda)u}_{L^2}
  \le
  \ip{u}{q_2(\mathcal A_\lambda)u}_{L^2},
  \qquad u\in\mathsf H.
\end{equation}
Furthermore, for $t\ge0$,
\begin{align}
  [r\mapsto e^{-tr}](\mathcal A_\lambda)&=P_t,
  \label{eq:calculus-exponential}\\
  q(\mathcal A_\lambda)1&=q(0)1.
  \label{eq:calculus-constant-mode}
\end{align}
For $t>0$,
\begin{equation}\label{eq:semigroup-regularization}
  P_t\mathsf H\subset D(\mathcal A_\lambda),
  \qquad
  \mathcal A_\lambda P_t
  =[r\mapsto r e^{-tr}](\mathcal A_\lambda),
  \qquad
  \norm{\mathcal A_\lambda P_t}_{L^2\to L^2}
  \le\frac1{et}.
\end{equation}
For $h>0$, define
\begin{equation}\label{eq:bh-def}
  b_h(r):=\int_0^h(1-e^{-sr})\dd s.
\end{equation}
Then, for every $u\in\mathsf H$,
\begin{equation}\label{eq:semigroup-integral-rule}
  b_h(\mathcal A_\lambda)u
  =\int_0^h(I-P_s)u\dd s,
\end{equation}
where the right-hand side is the limit in $L^2(\pi_\lambda)$ of Riemann
sums.  The mean-zero space
\[
  \mathsf H_0:=\left\{u\in\mathsf H:
  \int u\dd\pi_\lambda=0\right\}
\]
is invariant under $q(\mathcal A_\lambda)$.  Finally, the usual Poincar\'e
consequence of $m$-strong log-concavity gives
\begin{equation}\label{eq:L2-decay}
  \norm{P_tu}_{L^2(\pi_\lambda)}
  \le e^{-mt}\norm u_{L^2(\pi_\lambda)},
  \qquad u\in\mathsf H_0,\quad t\ge0.
\end{equation}
\end{proposition}

\begin{proof}
Write $A=\mathcal A_\lambda$.  The cited spectral theorem associates to
each $u\in\mathsf H$ a finite positive scalar measure $\nu_u$ on
$[0,\infty)$, of total mass $\norm u_{\mathsf H}^2$.  Its domain and norm
formulas, namely (A.4.1) and Theorem~A.4.2 of
\citet{BakryGentilLedoux2014}, read
\[
  D(F(A))=\left\{u\in\mathsf H:
     \int_{[0,\infty)}\abs{F(r)}^2\nu_u(\dd r)<\infty\right\},
  \qquad
  \norm{F(A)u}_{\mathsf H}^2
     =\int_{[0,\infty)}\abs{F(r)}^2\nu_u(\dd r).
\]
These formulas hold for real Borel functions $F$; the norm identity is
understood on $D(F(A))$.  The endpoint $0$ is included, with any mass it
carries.  For bounded $q$, the integral is at most
$\norm q_\infty^2\norm u_{\mathsf H}^2$, proving
$D(q(A))=\mathsf H$ and~\eqref{eq:borel-calculus}.  The real-valued
spectral construction is self-adjoint and obeys the algebra rules
\eqref{eq:calculus-linearity}--\eqref{eq:calculus-product}.
For the order rule, let $s=\sqrt{q_2-q_1}$.  Then
\[
  \ip{u}{(q_2-q_1)(A)u}_{\mathsf H}
  =\ip{u}{s(A)^2u}_{\mathsf H}
  =\norm{s(A)u}_{\mathsf H}^2\ge0.
\]
Equation~\eqref{eq:calculus-exponential} is simply
\eqref{eq:semigroup-from-A} in this notation.

Since $A1=0$, the norm formula gives
$\int r^2\nu_1(\dd r)=0$.  Hence $\nu_1$ is concentrated at $0$, and
\[
  \norm{q(A)1-q(0)1}_{\mathsf H}^2
  =\int_{[0,\infty)}\abs{q(r)-q(0)}^2\nu_1(\dd r)=0.
\]
This proves~\eqref{eq:calculus-constant-mode}.  Self-adjointness then yields
\[
  \int q(A)u\dd\pi_\lambda
  =\ip{u}{q(A)1}_{\mathsf H}
  =q(0)\int u\dd\pi_\lambda,
\]
so $q(A)$ preserves $\mathsf H_0$.

For~\eqref{eq:semigroup-regularization}, the spectral construction gives
$\nu_{P_tu}(\dd r)=e^{-2tr}\nu_u(\dd r)$.  Thus, for every $u\in\mathsf H$,
\[
  \int_{[0,\infty)}r^2\nu_{P_tu}(\dd r)
  =\int_{[0,\infty)}r^2e^{-2tr}\nu_u(\dd r)
  \le\frac{\norm u_{\mathsf H}^2}{e^2t^2}<\infty.
\]
The domain formula with $F(r)=r$ first proves $P_tu\in D(A)$;
multiplication in the same spectral representation then gives
$AP_tu=[r\mapsto re^{-tr}](A)u$ and its norm bound.
In particular, $\partial_tP_tu=-AP_tu$ in $\mathsf H$ for $t>0$, by the
semigroup-generator relation in \citet[Appendix~A.1]{BakryGentilLedoux2014}.

To prove~\eqref{eq:semigroup-integral-rule}, let
\[
  b_{h,n}(r):=\frac hn\sum_{j=0}^{n-1}(1-e^{-jhr/n}).
\]
Monotonicity of $s\mapsto1-e^{-sr}$ gives the uniform Riemann-sum estimate
\begin{equation}\label{eq:bh-riemann-uniform}
  0\le b_h(r)-b_{h,n}(r)
  \le\frac hn(1-e^{-hr})\le\frac hn.
\end{equation}
The algebra rules and~\eqref{eq:calculus-exponential} give
\[
  b_{h,n}(A)u=\frac hn\sum_{j=0}^{n-1}(I-P_{jh/n})u.
\]
By~\eqref{eq:borel-calculus}, the left-hand side converges to $b_h(A)u$.
Strong continuity makes the right-hand side converge in $\mathsf H$ to
the time integral in~\eqref{eq:semigroup-integral-rule}.
This is an integral of the continuous $\mathsf H$-valued function
$s\mapsto(I-P_s)u$, not an assertion of operator-norm continuity at $s=0$.
Bounded linear maps commute with these Riemann limits.  In particular,
stationary evaluation $v\mapsto v(X_{kh})$ is an isometry from
$L^2(\pi_\lambda)$ to $L^2(\Omega)$, and conditional expectation is an
$L^2$ contraction.  Thus this interpretation agrees with the use of the
integral in~\eqref{eq:D-conditional}.

For completeness,~\eqref{eq:L2-decay} is a Poincar\'e estimate, not an
additional rule of functional calculus.  Strong log-concavity gives
\[
  \Var_{\pi_\lambda}(u)
  \le\frac1m\int\norm{\nabla u}^2\dd\pi_\lambda.
\]
This holds first for smooth compactly supported $u$, and then for
$u\in L^2(\pi_\lambda)$ with a square-integrable weak gradient by cutoff
and smooth approximation.  An implication valid directly for the present
$C^1$ potential is given by
\citet[Theorem~2.2.12 and Corollary~2.1.2]{Chewi2026LogConcaveSampling}:
strong convexity bounds the Bregman cost in the former result below by
$m\norm{x-y}^2/2$, yielding the transport inequality with constant $1/m$;
the latter result gives the displayed Poincar\'e inequality.
Its equivalence with exponential variance decay is
\citet[Theorem~1.2.21]{Chewi2026LogConcaveSampling}, which gives
$\norm{P_tu}_{\mathsf H}^2\le e^{-2mt}\norm u_{\mathsf H}^2$ on
$\mathsf H_0$.  This proves~\eqref{eq:L2-decay} without requiring a
classical Hessian of $U_\lambda$ or a smoothing limit of $\pi_\lambda$.
\end{proof}

\begin{proof}[Proof of the corrector increment estimate~\eqref{eq:chi-increment}]
Write $b=\nabla U_\lambda$ and $A=\mathcal A_\lambda$, with all operators
acting coordinatewise.  Stationarity, reversibility, and the product rule give
\begin{align*}
  \E\norm{\chi_h(X_h)-\chi_h(X_0)}^2
  &=2\ip{\chi_h}{(I-P_h)\chi_h}_{L^2}\\
  &=2\ip{b}{[r\mapsto\varphi_h(r)^2(1-e^{-hr})](A)b}_{L^2}\\
  &\le2h^2\ip{b}{(I-P_h)b}_{L^2}
   =h^2\E\norm{b(X_h)-b(X_0)}^2.
\end{align*}
The inequality is~\eqref{eq:calculus-order} applied to
$0\le\varphi_h(r)^2(1-e^{-hr})\le h^2(1-e^{-hr})$,
using~\eqref{eq:phi-bounds}.
\end{proof}

\subsection{The single \texorpdfstring{$L^p$}{Lp} multiplier theorem used in the proof}

The bounded functional calculus above is an $L^2$ result.  To control all
moments of the corrector, we need one $L^p$ estimate.  We first state the
external result in the exact notation needed to check its assumptions.

For $1<p<\infty$, let $\mathcal A_{\lambda,p}$ denote the generator of
the same semigroup $(P_t)$ acting on $L^p(\pi_\lambda)$; explicitly,
\[
  D(\mathcal A_{\lambda,p})
  :=\left\{u\in L^p(\pi_\lambda):
  \lim_{t\downarrow0}\frac{u-P_tu}{t}
  \text{ exists in }L^p(\pi_\lambda)\right\},
\]
and the limit is $\mathcal A_{\lambda,p}u$.  Its range and null space are
\[
  R(\mathcal A_{\lambda,p})
  :=\{\mathcal A_{\lambda,p}u:u\in D(\mathcal A_{\lambda,p})\},
  \qquad
  N(\mathcal A_{\lambda,p})
  :=\{u\in D(\mathcal A_{\lambda,p}):\mathcal A_{\lambda,p}u=0\}.
\]
The bar in $\overline{R(\mathcal A_{\lambda,p})}^{\,L^p}$ means closure in the $L^p$ norm.

\begin{proposition}[Laplace-transform multiplier]
\label{prop:CD-multiplier}
Let $(\Omega,\nu)$ be a $\sigma$-finite measure space, let $A$ be a
nonnegative self-adjoint operator on $L^2(\nu)$, and suppose that
$T_t=e^{-tA}$ satisfies
\[
  \norm{T_tf}_{L^q(\nu)}\le\norm f_{L^q(\nu)},
  \qquad 1\le q\le\infty.
\]
For $M\in L^\infty(0,\infty)$ define, for $r>0$,
\begin{equation}\label{eq:CD-Laplace-multiplier}
  \widetilde M(r)
  :=r\int_0^\infty M(s)e^{-sr}\dd s.
\end{equation}
Then, for $1<p<\infty$ and
$f\in\overline{R(A_p)}^{\,L^p}$,
\begin{equation}\label{eq:CD-bound}
  \norm{\widetilde M(A)f}_{L^p(\nu)}
  \le120(p^*-1)\norm M_\infty\norm f_{L^p(\nu)},
  \qquad
  p^*:=\max\left\{p,\frac p{p-1}\right\}.
\end{equation}
Moreover, let $\Pi_0^{(2)}$ denote the orthogonal projection of
$L^2(\nu)$ onto $N(A_2)$.  Then
\begin{equation}\label{eq:CD-decomposition}
  \Pi_0^{(2)}f=\lim_{t\to\infty}T_tf
  \quad\text{in }L^2(\nu).
\end{equation}
For every $1<p<\infty$, the restriction of $\Pi_0^{(2)}$ to
$L^2(\nu)\cap L^p(\nu)$ extends to a contractive projection
$\Pi_0^{(p)}$ on $L^p(\nu)$.  Thus, by the meaning of extension,
\[
  \Pi_0^{(p)}f=\Pi_0^{(2)}f,
  \qquad
  f\in L^2(\nu)\cap L^p(\nu).
\]
Moreover,
\[
  (I-\Pi_0^{(p)})f
  \in\overline{R(A_p)}^{\,L^p},
  \qquad
  f\in L^2(\nu)\cap L^p(\nu).
\]
\end{proposition}

\begin{proof}[Source and notation]
These statements are quoted, not reproved.  The hypotheses that $A$ is nonnegative and self-adjoint and that $e^{-tA}$ is contractive on every $L^q$ are exactly the standing assumptions in \citet[Introduction]{CarbonaroDragicevic2017}; that paper calls this a \emph{symmetric contraction semigroup}.  Its notation $A_p,D(A_p),R(A_p),N(A_p)$ is introduced in \citet[Section~2]{CarbonaroDragicevic2017}. 

The limit formula~\eqref{eq:CD-decomposition} is \citet[Lemma~2]{CarbonaroDragicevic2017}, and the contractive $L^p$ extension is part~(i) of that lemma.  The cited paper uses the same symbol $P_0$ for the $L^2$ projection and all of its $L^p$ extensions; the superscripts in our notation only indicate the ambient space.  The fact that
\[
  (I-\Pi_0^{(p)})f
  \in\overline{R(A_p)}^{\,L^p}
\]
is the form used explicitly in \citet[proof of Theorem~1]{CarbonaroDragicevic2017}.

The definition \eqref{eq:CD-Laplace-multiplier} is the display immediately preceding \citet[Proposition~5]{CarbonaroDragicevic2017}, and the estimate \eqref{eq:CD-bound} is \citet[Proposition~10]{CarbonaroDragicevic2017}.
\end{proof}

We now verify that the proposition applies here.  The measure
$\pi_\lambda$ is a probability measure and hence is $\sigma$-finite.
Nonnegativity and self-adjointness of $\mathcal A_\lambda$ follow from
\Cref{prop:generator-framework}; contractivity on every $L^q$ is
\eqref{eq:Lp-contraction}. If
$v\in L^2(\pi_\lambda)\cap L^p(\pi_\lambda)$ has zero mean, then
\eqref{eq:L2-decay} gives $P_tv\to0$ in $L^2$.  Taking $T_t=P_t$ in~\eqref{eq:CD-decomposition} therefore gives $\Pi_0^{(2)}v=0$. 
Since $\Pi_0^{(p)}$ extends $\Pi_0^{(2)}$ and
$v\in L^2(\pi_\lambda)\cap L^p(\pi_\lambda)$,
\[
  \Pi_0^{(p)}v=\Pi_0^{(2)}v=0.
\]
Consequently,
\begin{equation}\label{eq:mean-zero-in-range}
  v=(I-\Pi_0^{(p)})v
  \in\overline{R(\mathcal A_{\lambda,p})}^{\,L^p}.
\end{equation}
Thus every hypothesis of \Cref{prop:CD-multiplier} has now been checked in
the notation of the present paper.

\begin{proposition}[Proof of the corrector $L^p$ multiplier bound]
\label{prop:psi-Lp-appendix}
The estimate~\eqref{eq:psi-Lp} holds.
\end{proposition}

\begin{proof}
Define the bounded sawtooth function
\[
  \omega(t):=\lceil t\rceil-t,
  \qquad
  \bar\omega(t):=\omega(t)-\frac12.
\]
The values at the integers are irrelevant, and
$\norm{\bar\omega}_\infty\le1/2$.  For $u>0$, summing over the intervals
$[n,n+1)$ gives
\begin{align}
  u\int_0^\infty\omega(t)e^{-ut}\dd t
  &=u\sum_{n=0}^\infty e^{-nu}
    \int_0^1(1-s)e^{-us}\dd s
  \notag\\
  &=\frac1{1-e^{-u}}-\frac1u
  =\psi(u).
  \label{eq:psi-sawtooth}
\end{align}
Since
$u\int_0^\infty \frac12e^{-ut}\dd t=1/2$ for $u>0$,
\begin{equation}\label{eq:psi-centered-laplace}
  \psi(u)
  =\frac12+u\int_0^\infty\bar\omega(t)e^{-ut}\dd t,
  \qquad u>0.
\end{equation}
The restriction $u>0$ is essential: at $u=0$ the integral by itself is not
convergent.  We instead use the separately defined value
$\psi(0)=1/2$ from~\eqref{eq:psi-def}.

For $h>0$, set
\[
  M_h(s):=\bar\omega(s/h),
  \qquad
  \norm{M_h}_\infty\le\frac12,
\]
and define
\begin{equation}\label{eq:mh-def}
  m_h(r):=
  \begin{cases}
  \displaystyle r\int_0^\infty M_h(s)e^{-sr}\dd s,&r>0,\\[2mm]
  0,&r=0.
  \end{cases}
\end{equation}
The change of variables $s=ht$ in
\eqref{eq:psi-centered-laplace} gives the pointwise identity
\begin{equation}\label{eq:psi-laplace-scaled}
  \psi(hr)=\frac12+m_h(r),
  \qquad r\ge0.
\end{equation}
Notice that~\eqref{eq:psi-laplace-scaled} is now valid also at $r=0$,
because $m_h(0)$ was explicitly defined to be zero.

Let $v\in L^2\cap L^p$ have zero mean.  By
\eqref{eq:mean-zero-in-range}, \Cref{prop:CD-multiplier} applies to
$M_h$.  Since $p\ge2$ implies $p^*=p$, it gives
\[
  \norm{m_h(\mathcal A_\lambda)v}_{L^p}
  \le120(p-1)\norm{M_h}_\infty\norm v_{L^p}
  \le60(p-1)\norm v_{L^p}.
\]
Using~\eqref{eq:psi-laplace-scaled}, first in the $L^2$ functional calculus
and then in the $L^p$ extension supplied by the preceding estimate,
\[
  \psi(h\mathcal A_\lambda)v
  =\frac12v+m_h(\mathcal A_\lambda)v.
\]
Consequently,
\[
  \norm{\psi(h\mathcal A_\lambda)v}_{L^p}
  \le\left(\frac12+60(p-1)\right)\norm v_{L^p}
  \le Cp\norm v_{L^p},
\]
which is~\eqref{eq:psi-Lp}.
\end{proof}

\bibliographystyle{plainnat}
\bibliography{reference}

\begin{thebibliography}{13}
\providecommand{\natexlab}[1]{#1}
\providecommand{\url}[1]{\texttt{#1}}
\expandafter\ifx\csname urlstyle\endcsname\relax
  \providecommand{\doi}[1]{doi: #1}\else
  \providecommand{\doi}{doi: \begingroup \urlstyle{rm}\Url}\fi

\bibitem[Bakry et~al.(2014)Bakry, Gentil, and Ledoux]{BakryGentilLedoux2014}
Dominique Bakry, Ivan Gentil, and Michel Ledoux.
\newblock \emph{Analysis and Geometry of Markov Diffusion Operators}.
\newblock Springer, Cham, 2014.
\newblock \doi{10.1007/978-3-319-00227-9}.

\bibitem[Bauschke and Combettes(2017)]{BauschkeCombettes2017}
Heinz~H. Bauschke and Patrick~L. Combettes.
\newblock \emph{Convex Analysis and Monotone Operator Theory in Hilbert Spaces}.
\newblock Springer, Cham, 2 edition, 2017.
\newblock \doi{10.1007/978-3-319-48311-5}.

\bibitem[Carbonaro and Dragi\v{c}evi\'{c}(2017)]{CarbonaroDragicevic2017}
Andrea Carbonaro and Oliver Dragi\v{c}evi\'{c}.
\newblock Functional calculus for generators of symmetric contraction semigroups.
\newblock \emph{Duke Mathematical Journal}, 166\penalty0 (5):\penalty0 937--974, 2017.
\newblock \doi{10.1215/00127094-3774526}.

\bibitem[Chewi(2026)]{Chewi2026LogConcaveSampling}
Sinho Chewi.
\newblock Log-concave sampling.
\newblock \url{https://chewisinho.github.io/main.pdf}, 2026.
\newblock Book manuscript.

\bibitem[Durmus et~al.(2018)Durmus, Moulines, and Pereyra]{DurmusMoulinesPereyra2018}
Alain Durmus, {\'{E}}ric Moulines, and Marcelo Pereyra.
\newblock Efficient bayesian computation by proximal markov chain monte carlo: When langevin meets moreau.
\newblock \emph{SIAM Journal on Imaging Sciences}, 11\penalty0 (1):\penalty0 473--506, 2018.
\newblock \doi{10.1137/16M1108340}.

\bibitem[Durmus et~al.(2019)Durmus, Majewski, and Miasojedow]{DurmusMajewskiMiasojedow2019}
Alain Durmus, Szymon Majewski, and B{\l}a{\.z}ej Miasojedow.
\newblock Analysis of {Langevin Monte Carlo} via convex optimization.
\newblock \emph{Journal of Machine Learning Research}, 20\penalty0 (73):\penalty0 1--46, 2019.
\newblock URL \url{https://jmlr.org/papers/v20/18-173.html}.

\bibitem[Evans and Gariepy(2015)]{EvansGariepy2015}
Lawrence~C. Evans and Ronald~F. Gariepy.
\newblock \emph{Measure Theory and Fine Properties of Functions}.
\newblock CRC Press, Boca Raton, revised edition, 2015.
\newblock \doi{10.1201/b18333}.

\bibitem[Leoni(2017)]{Leoni2017}
Giovanni Leoni.
\newblock \emph{A First Course in Sobolev Spaces}, volume 181 of \emph{Graduate Studies in Mathematics}.
\newblock American Mathematical Society, Providence, RI, 2 edition, 2017.
\newblock \doi{10.1090/gsm/181}.

\bibitem[Mou et~al.(2022)Mou, Flammarion, Wainwright, and Bartlett]{MouEtAl2022}
Wenlong Mou, Nicolas Flammarion, Martin~J. Wainwright, and Peter~L. Bartlett.
\newblock An efficient sampling algorithm for non-smooth composite potentials.
\newblock \emph{Journal of Machine Learning Research}, 23\penalty0 (233):\penalty0 1--50, 2022.
\newblock URL \url{https://jmlr.org/papers/v23/20-527.html}.

\bibitem[Pedrotti and Whalley(2026)]{PedrottiWhalley2026}
Francesco Pedrotti and Peter~A. Whalley.
\newblock Wasserstein mixing time of the {Unadjusted Langevin Algorithm}.
\newblock arXiv:2608.02430, 2026.

\bibitem[Pereyra(2016)]{Pereyra2016}
Marcelo Pereyra.
\newblock Proximal {Markov} chain {Monte Carlo} algorithms.
\newblock \emph{Statistics and Computing}, 26\penalty0 (4):\penalty0 745--760, 2016.
\newblock \doi{10.1007/s11222-015-9567-4}.

\bibitem[Salim and Richt{\'a}rik(2020)]{SalimRichtarik2020}
Adil Salim and Peter Richt{\'a}rik.
\newblock Primal dual interpretation of the proximal stochastic gradient {Langevin} algorithm.
\newblock In \emph{Advances in Neural Information Processing Systems}, volume~33, 2020.
\newblock URL \url{https://papers.nips.cc/paper/2020/hash/2779fda014fbadb761f67dd708c1325e-Abstract.html}.

\bibitem[Xin and Zhang(2026)]{XinZhang2026}
Yuchen Xin and Zhihua Zhang.
\newblock Active-trace complexity bounds for {Moreau--Yosida Unadjusted Langevin Sampling}.
\newblock arXiv:2608.13467, 2026.

\end{thebibliography}

\end{document}